\documentclass{article}
\usepackage{arxiv_preprint,times}
\usepackage{amsmath,amsfonts,bm}

\def\eqref#1{equation~\ref{#1}}
\def\1{\bm{1}}

\DeclareMathAlphabet{\mathsfit}{\encodingdefault}{\sfdefault}{m}{sl}
\SetMathAlphabet{\mathsfit}{bold}{\encodingdefault}{\sfdefault}{bx}{n}

\newcommand{\E}{\mathbb{E}}

\newcommand{\Var}{\mathrm{Var}}

\newcommand{\Cov}{\mathrm{Cov}}
\usepackage{amssymb,amsthm}
\usepackage{booktabs}
\usepackage{graphicx}
\usepackage{wrapfig}
\usepackage{needspace}
\makeatletter
\newcommand{\finishwrap}{%
  \par
  \ifnum\c@WF@wrappedlines>\@ne
    \@tempdima=\baselineskip
    \@tempcnta=\c@WF@wrappedlines
    \advance\@tempcnta\m@ne
    \multiply\@tempdima\@tempcnta
    \vskip\@tempdima
  \fi
  \WFclear
}
\makeatother
\usepackage{caption}
\usepackage{subcaption}
\usepackage{array}
\usepackage{xcolor}
\usepackage{colortbl}
\definecolor{tableaccent}{RGB}{238,247,245}
\definecolor{keyred}{RGB}{160,30,35}
\usepackage{algorithm}
\usepackage{algorithmic}
\usepackage[hidelinks]{hyperref}
\usepackage{url}

\newtheorem{proposition}{Proposition}
\newtheorem{lemma}{Lemma}
\newtheorem{theorem}{Theorem}
\newtheorem{corollary}{Corollary}

\newcommand{\tfour}{\ensuremath{\mathrm{T}^{5}}}
\newcommand{\qstar}{\textsc{Quiet-STaR}}
\newcommand{\clip}{\mathrm{clip}}
\newcommand{\one}{\mathbf{1}}
\renewcommand{\eqref}[1]{\textup{(\ref{#1})}}

\newif\ifshortconclusion
\shortconclusiontrue

\title{\tfour: Twin-Critic Training for Token-Level Thoughts in Reinforcement Mid-Training}

\author{%
\begin{minipage}{0.96\textwidth}
\centering\normalfont
Nan Qiao\textsuperscript{1,2}\thanks{Part of this work was done when Nan Qiao worked at Tencent.},
Yebin Yang\textsuperscript{3},
Weinong Wang\textsuperscript{2,\textdagger},
Shuning Wang\textsuperscript{4},
Shangpin Peng\textsuperscript{5},
Fengyuan Lu\textsuperscript{6},
Xinming Wang\textsuperscript{7},
Zhehan Kan\textsuperscript{1},
Ruixu Zhang\textsuperscript{1},
Songyang Zhang\textsuperscript{2},
Sheng Yue\textsuperscript{8},\\
Yonglong Tian\textsuperscript{1,\textdagger},
Ju Ren\textsuperscript{1}\\[0.8em]
\small
\textsuperscript{1}Tsinghua University,\quad
\textsuperscript{2}Tencent,\quad
\textsuperscript{3}Shanghai Jiao Tong University,\quad
\textsuperscript{4}Central South University,\quad\\
\textsuperscript{5}The Hong Kong University of Science and Technology,\quad
\textsuperscript{6}Nanjing University,\quad\\
\textsuperscript{7}Institute of Automation, Chinese Academy of Sciences,\quad
\textsuperscript{8}Sun Yat-sen University,\quad\\
\textsuperscript{\textdagger}Corresponding authors
\end{minipage}
}
\date{}

\preprintfinalcopy
\begin{document}

\maketitle

\begin{abstract}
Reinforcement mid-training lets language models learn internal thoughts from unlabeled text, but efficient token-level credit assignment remains challenging.  Existing group-relative methods require costly repeated generation.  Learned critics offer single-rollout feedback, but accurate return prediction alone does not ensure reliable policy updates.  Our analysis shows how training--inference mismatch and PPO clipping prevent a common offset in advantage estimates from cancelling out, introducing additional update drift.  We propose \tfour{}, a twin-critic method that calibrates token-level advantages from a single generated trajectory.  After warmup and held-out qualification, the critics provide two advantage estimates, combined using action-dependent weights learned through a conditional-moment saddle-point objective.  This objective brings the average advantage at each prefix toward zero, while a signal-retention constraint prevents the correction from erasing the learning signal.  Sharing information across text positions avoids repeated sampling of each prefix.  Theoretically, we characterize optimal mixing under the signal-retention constraint and establish an upper bound on residual mean-induced drift.  Experiments show that, compared with the state-of-the-art critic-free method, \tfour{} improves mean benchmark performance by 7.8\% and reduces mean training-step time by up to 63.4\%.
\end{abstract}

\section{Introduction}

Language models can learn to reason from ordinary text by generating internal thoughts that help predict what comes next.  In \qstar{}, improved continuation likelihood rewards these thoughts without labeled answers or reasoning traces \citep{zelikman2024quietstar}.  Recent work brings reinforcement learning into pre- and mid-training: RPT and RLPT turn next-token and next-segment prediction into reasoning tasks \citep{dong2025reinforcementpretraining,li2025rlpretrainingdata}.  RLP rewards the information gained from thoughts, while PretrainZero learns to select and predict masked spans \citep{hatamizadeh2025rlp,xing2025pretrainzero}.  RMT targets the intermediate training stage with adaptive thought-length limits and curriculum sampling \citep{tian2025rmt}.  A central challenge remains token-level credit assignment: a useful continuation does not reveal which thought tokens contributed.

Providing this token-level feedback, however, can require substantial generation.  Critic-free group-relative methods such as GRPO use multiple responses to construct a baseline \citep{shao2024deepseekmath}, as do RMT and RLP \citep{tian2025rmt,hatamizadeh2025rlp}.  We use groups of eight throughout our GRPO comparisons.  A learned critic instead shares return predictions across texts and prefixes, as explored in value-based LLM optimization \citep{yue2025vapo}.  By predicting the remaining return from a partial thought, it can combine prefix rewards with generalized advantage estimation (GAE) to provide token-level feedback along a single trajectory \citep{schulman2016gae,schulman2017ppo}.

Efficient single-rollout training also has to accommodate discrepancies between generation and optimization.  Reusing rollouts introduces policy lag \citep{schulman2017ppo}, while differences in numerical precision, kernels, and batching can make training and inference probabilities disagree even at the same model parameters \citep{marek2026score}.  Stabilization methods address different aspects of these discrepancies: GSPO uses sequence-level reweighting \citep{zheng2025gspo}, CPPO constrains position-dependent and accumulated prefix drift \citep{mao2026cppo}, and Score Centering corrects the biased expectation of the policy score---the log-probability gradient---under sampling mismatch \citep{marek2026score}.  These methods regulate sampling and policy updates.  A learned critic must also predict current returns reliably.  Newly initialized value heads therefore need warmup and checks on current, held-out returns before guiding actor updates.

\textbf{\textit{However, reliable return predictions alone do not remove the effect of training--inference mismatch on advantage-based updates.}}  At a given thought prefix, advantage estimates may rank actions correctly yet be uniformly too high or too low.  This common offset cancels from the expected policy gradient under matched sampling without clipping, but mismatch or PPO's sign-dependent clipping can prevent that cancellation.  Our analysis separates the expected update into a component reflecting variation across actions and a mean-induced drift component.  The latter couples the advantage offset to the average update direction after importance weighting and clipping.  This identifies an advantage-side complement to the sampling and score corrections above: bring the prefix-conditioned advantage mean toward zero without erasing the learning signal.

To control this mean-induced drift, we propose \tfour{}, a twin-critic method that learns how to combine two advantage estimates.  After warmup and held-out qualification, both critics evaluate the same generated thought.  An action-dependent weight mixes their estimates at each token.  We learn this weight through a conditional-moment saddle-point objective: a maximizing auxiliary mean predictor exposes the mixture's residual offset, while the minimizing weight uses critic disagreement to reduce it.  The mean predictor learns across contexts, using one thought per selected position for both estimates rather than repeated prefix continuations.  A constraint limits deviation from the fixed twin average to preserve nonzero aggregate signal.  This saddle-point formulation links calibration to optimal mixing and a bound on residual mean-induced drift.

Our contributions are this drift analysis, a prediction-qualified single-rollout calibration method, and its theoretical guarantees.  In the population formulation, we prove strong duality and characterize optimal mixing under the signal-retention constraint.  We also bound retained signal strength and residual mean-induced drift, accounting for mean-prediction error and validation uncertainty.  Compared with RLP, the state-of-the-art critic-free reinforcement pretraining method, \tfour{} improves mean benchmark performance by 7.8\% and reduces mean training-step time by up to 63.4\%.

\section{Related Work}

\paragraph{Latent reasoning and reinforcement mid-training.}
\qstar{} learns hidden rationales from their effect on future-token likelihood \citep{zelikman2024quietstar}, while Fast Quiet-STaR compresses explicit thought tokens \citep{huang2025fast}.  These methods connect intermediate computation to ordinary text, complementing reasoning-trace bootstrapping such as STaR \citep{zelikman2022star}.  Reinforcement Pre-Training, RL on Pre-Training Data, and RLP also construct reinforcement signals from pretraining corpora \citep{dong2025reinforcementpretraining,li2025rlpretrainingdata,hatamizadeh2025rlp}.  PretrainZero selects masked spans for self-supervised reinforcement pretraining \citep{xing2025pretrainzero}.  RMT combines reinforcement mid-training with thought-token allocation, curriculum sampling, and next-token prediction \citep{tian2025rmt}.  Our task follows \qstar{}'s continuation-utility objective at this intermediate training stage.  The focus is whether learned token-value baselines can supply reliable PPO advantages across the resulting thought prefixes.

\paragraph{Group-relative estimation and token-level actor--critic learning.}
GRPO replaces the learned value baseline with comparisons among responses sampled for a prompt \citep{shao2024deepseekmath}, a route also used for large-scale reasoning reinforcement learning \citep{deepseekai2025r1}.  A learned critic instead shares information across sampled states and supports GAE along each trajectory \citep{schulman2016gae,schulman2017ppo}.  VAPO studies value-based LLM optimization and introduces the length-adaptive trace used here \citep{yue2025vapo}.  SAO learns baselines under single-rollout asynchronous agentic training \citep{hou2026sao}.  Related offline-RL work addresses critic-side instability by controlling harmful TD cross-covariance or modifying optimizer dynamics to suppress critic collapse \citep{qiao2026less,qiao2026adamoptimizer}.  Our states are hidden thought prefixes and rewards come from continuation prediction, with a fixed rollout law within each window.  Twin critics in TD3 and SAC manage function-approximation error through action-value targets \citep{fujimoto2018addressing,haarnoja2018soft}.  Our two state-value critics fit observed returns, and a conditional-moment objective learns how to combine their advantages.

\paragraph{What policy stabilization controls.}
Existing methods act on different parts of the update.  TRPO constrains policy displacement and PPO clips probability ratios \citep{schulman2015trpo,schulman2017ppo}.  REINFORCE++ and Dr.~GRPO change advantage or loss normalization \citep{hu2025reinforceplusplus,liu2025understanding}.  DAPO combines clipping and sampling changes with token-level loss aggregation \citep{yu2025dapo}.  GSPO uses sequence-level ratios, while CPPO makes the trust region position- and prefix-dependent \citep{zheng2025gspo,mao2026cppo}.  Score Centering addresses update drift under training--inference engine mismatch through an additive score correction \citep{marek2026score}.  \tfour{} retains clipped PPO and acts on the conditional advantage mean, after qualifying the predictors that produce it.  This distinction respects the usual baseline-cancellation result: inaccurate state baselines can still cancel under exact unclipped score weighting \citep{williams1992reinforce,sutton2000policy}.  The issue here is their interaction with bootstrapping and incomplete score cancellation.

\section{Preliminaries and Problem Setup}
\label{sec:problem-setup}

\subsection{Reinforcement Mid-Training from Unlabeled Text}
Reinforcement mid-training uses ordinary text to supervise useful internal computation before task-specific post-training \citep{tian2025rmt,hatamizadeh2025rlp}.  From a Base model, we sample a $T$-token hidden thought $z_{1:T}$ at position $p$ in a length-$S$ text $x_{1:S}$ drawn from corpus $\mathcal D_{\mathrm{mid}}$.  Following \qstar{} \citep{zelikman2024quietstar}, we score the thought by its effect on prediction loss over the next $H$ observed text tokens.  After $t$ thought tokens, this loss is
\begin{equation}
\ell_t=-\frac1H\sum_{j=1}^H\log p_{\mathrm{score}}\!\left(x_{p+j}\mid x_{\le p},z_{1:t},x_{p+1:p+j-1}\right),
\label{eq:continuation-ce}
\end{equation}
where $p_{\mathrm{score}}$ is the scoring model's token distribution, and $\ell_0$ is the loss without a thought.  The full thought has utility $\mathcal U=\ell_0-\ell_T$: it compares prediction of the same continuation with and without the thought.  Each observed token contributes a log-likelihood gain.  The changes after successive thought tokens also add up to $\mathcal U=\sum_{t=1}^{T}(\ell_{t-1}-\ell_t)$.  An individual change may be negative even when the full thought helps.

The actor samples token $a_t=z_t$ from the visible prefix $s_t=(x_{\le p},z_{<t})$.  Actor and critic inputs exclude the scoring text.  To assign rewards along the thought, write $\Psi_t=\ell_0-\ell_t$ for the gain after its first $t$ tokens.  We scale and clip this gain, then reward each change in the resulting potential:
\begin{equation}
\Phi_t=\clip(\Psi_t/\sigma_r,-c_r,c_r),\qquad
r_t=\Phi_t-\Phi_{t-1},\quad\Phi_0=0,
\label{eq:dense-reward}
\end{equation}
where $\sigma_r>0$ is the reward scale and $c_r>0$ the clipping threshold.  Between scoring checkpoints we carry the last potential forward, so $\sum_t r_t=\Phi_T=\clip(\mathcal U/\sigma_r,-c_r,c_r)$.  The scorer and reward scale stay frozen within each training window (Appendix~\ref{app:reward-details}).  For analysis, $c$ denotes the full pre-action context, including the visible prefix $s$, scoring text, and checkpoint history.  Unless stated otherwise, conditioning on $s$ also fixes these other parts of $c$.  This shorthand does not change what the actor and critic observe.

\paragraph{Training and prediction.}
\label{sec:prediction-interface}
Scoring and critic evaluation occur during training.  A gate $\beta_g(s)\in[0,1]$ combines thought-conditioned and no-thought predictions, defining the mixed gain in Figure~\ref{fig:learning-dynamics}.  Ordinary next-token loss anchors language modeling.  This prediction gate is separate from the advantage weight in Section~\ref{sec:twin-calibration}.  Its objective detaches expert logits and features (Appendix~\ref{app:gate-behavior}).  Scorer and generation settings are specified with the evaluation protocol (Appendix~\ref{app:experimental-details}).

\subsection{PPO with a Learned Token-Level Critic}
\label{sec:actor-critic-setup}
\label{sec:policy-mismatch}
A learned critic shares return information across texts and partial thoughts.  PPO combines this reusable baseline with the prefix rewards above, which can also support critic-free estimators.
Let $\pi_\theta$ be the actor with parameters $\theta$, $\pi_{\mathrm{old}}$ the recorded rollout policy, and $q$ the actual sampling law.  A critic $V_\psi(s_t)$, with parameters $\psi$, predicts remaining return $G_t=\sum_{k=0}^{T-t}\gamma^k r_{t+k}$ with discount $\gamma$.  Generalized advantage estimation (GAE) gives \citep{schulman2016gae}
\begin{equation}
\delta_t=r_t+\gamma V_\psi(s_{t+1})-V_\psi(s_t),\qquad
A_t=\delta_t+\gamma\lambda A_{t+1},
\label{eq:gae}
\end{equation}
where $\delta_t$ is the temporal-difference residual, $A_t$ the advantage estimate, and $\lambda$ the trace parameter controlling how far later residuals propagate.  We use $\gamma=1$, full-return targets, and a length-adaptive trace (Appendix~\ref{app:advantage-snapshots}).  Values and advantages vanish after termination.

\Needspace{8\baselineskip}
With $\widetilde A_t$ denoting a normalized, detached advantage, PPO maximizes \citep{schulman2017ppo}
\begin{equation}
\mathcal J_{\mathrm{clip}}(\theta)=\E_{\mathcal D}\!\left[\min\{\rho_t\widetilde A_t,\clip(\rho_t,\rho_-,\rho_+)\widetilde A_t\}\right],
\label{eq:ppo}
\end{equation}
where $\rho_t=\pi_\theta(a_t\mid s_t)/\pi_{\mathrm{old}}(a_t\mid s_t)$ is the token probability ratio and $0<\rho_-\leq1\leq\rho_+$ are its lower and upper clipping bounds, which may be asymmetric.  Replay $\mathcal D$ averages valid actor actions.  Recorded probabilities, critic snapshots, and advantages stay fixed during actor updates.

Unlike the learned predictor $V_\psi$, $V^q(s)=\E_q[G_t\mid s_t=s]$ denotes the conditional expected return under $q$, and $V^{\pi_\theta}$ is its current-actor counterpart.  Policy lag is movement from $\pi_{\mathrm{old}}$ to $\pi_\theta$, while sampling mismatch is a difference between $q$ and $\pi_{\mathrm{old}}$ \citep{marek2026score}.  Critic mismatch concerns prediction relative to $V^q$, including stale fitting targets or hidden scoring information (Appendix~\ref{app:information-limit}).  Refreshing the rollout copy does not refit the critic.  GRPO instead uses current group rewards, avoiding learned-value history but retaining policy lag and finite-group estimation error.

\paragraph{Statistical notation.}
For any scalar or vector quantity $X$, write $\mu_X(s)=\E_q[X\mid s]$ for its conditional mean.  The symbols $\Var_q$ and $\Cov_q$ denote variance and covariance.  We omit token index $t$ when discussing a generic prefix.  The finite measure $d$ weights rollout prefixes, and $\E_d$ denotes the resulting aggregation, while $\E_{d,q}$ additionally averages sampled actions and continuations.  These weights need not sum to one, but their scale is fixed across the learning objectives (Appendix~\ref{app:restart-contract}).

\section{Critic Mismatch and Predictive Qualification}
\label{sec:predictive-reliability}

\subsection{The Critic Adds Its Own Mismatch}
\label{sec:critic-update-mismatch}
Clipping constrains how far the actor moves from its recorded policy, but it does not tell us whether the value reference still describes current continuation returns.  Keeping the critic aligned with a changing policy is a longstanding actor--critic concern \citep{konda2003actorcritic,yue2025vapo}.  Let $q_\beta$ denote the earlier rollout law that supplied the critic's fitting data and $q$ the current rollout law.  With the scoring rule fixed,
\[
V_\psi-V^{\pi_\theta}
=\underbrace{(V_\psi-V^{q_\beta})+(V^{q_\beta}-V^q)}_{\text{fit error $+$ target lag}}
+\underbrace{(V^q-V^{\pi_\theta})}_{\text{rollout--actor gap}}.
\]
The first two terms arise from fitting and reusing a learned critic.  Even under an idealized refresh with $q=\pi_{\mathrm{old}}=\pi_\theta$ and $\rho=1$, the rollout--actor gap vanishes while these critic-specific terms can remain.  GRPO avoids this learned-critic history by constructing its baseline from the sampled group \citep{shao2024deepseekmath}.  Critic-based PPO therefore needs value tracking in addition to policy-ratio control.  This extra mismatch matters because the critic supplies the reference level for every thought token.  The advantage asks whether a continuation exceeds its expected remaining return, not merely whether its immediate reward is positive.  In a long thought, a small local gain may lead to a promising prefix, while a large gain may leave little useful continuation.  GAE then carries later value errors back to earlier tokens \citep{schulman2016gae}, so a stale reference can distort credit assignment throughout the thought (Appendix~\ref{app:value-gradient-proof}).  Before using critic-based advantages, we therefore test both critics on unseen returns from the current rollout law.

\subsection{Training and Qualifying Both Critics}
\label{sec:critic-collapse}
\label{sec:prediction-gated-training}
\label{sec:naive-ppo-diagnostic}
To control the critic-specific mismatch above, \tfour{} trains and qualifies both critics in three steps: stabilize the current return target, test each critic on unseen returns, and repeat the test as the actor changes.

\paragraph{Preparing a stable return target.}
Within each qualification window, \tfour{} freezes the actor and its rollout copy, scorer, and reward scale so that two separately parameterized critics $V_i=V_{\psi_i}$, $i\in\{1,2\}$, face a stable prediction target.  We warm up their value heads before adapting the full critics and reset qualification when the backbones are unfrozen.  Prediction quality, rather than a fixed warmup duration, determines readiness.

\paragraph{Testing both critics.}
Let $\mathcal D_{\mathrm{qual}}=\{(s_j,G_j)\}_{j=1}^{N_h}$ contain $N_h$ held-out prefix--return observations with nonzero return variance.  Each critic receives the predictive score
\begin{equation}
R_i^2=1-\frac{\sum_j(G_j-V_i(s_j))^2}{\sum_j(G_j-\bar G)^2},
\label{eq:holdout-r2}
\end{equation}
where $\bar G$ is the holdout mean.  Actor updates require $\min(R_1^2,R_2^2)\geq\eta$, with $0<\eta<1$: both critics must reduce the constant predictor's error by at least a fraction $\eta$.  A score of zero matches that predictor, one fits all held-out returns, and a negative score is worse.  Rescaling returns and predictions together leaves the score unchanged.

To see why useful prediction need not have zero loss, write $G-V_i=(G-V^q)+(V^q-V_i)$ at a fixed prefix and scoring item.  Since $\E_q[G-V^q\mid s]=0$, the cross term vanishes:
\begin{equation}
\E_q[(G-V_i(s))^2\mid s]
=\Var_q(G\mid s)+(V_i(s)-V^q(s))^2,
\label{eq:prediction-risk-split}
\end{equation}
where the first term is continuation variability and the second is error in its conditional mean.  Qualification tests predictive improvement despite that variability.  Taking the weaker score prevents a strong critic from hiding an uninformative partner (Appendix~\ref{app:predictive-qualification}).

\paragraph{Rechecking as the actor changes.}
Qualification requires $m\geq1$ consecutive fresh tests on data excluded from fitting.  Checks continue at fixed intervals because expected returns after familiar prefixes can change as the actor learns.  Sustained failure pauses the actor and resumes critic fitting until qualification is restored.  Pausing stabilizes the target distribution, not individual sampled returns, while consecutive passes check that improvement persists across batches.  Changes to rollout or scoring targets trigger new checks.  Appendix~\ref{app:data-contracts} gives the data separation and update schedule.

\Needspace{18\baselineskip}
\begin{wrapfigure}{r}{0.4\textwidth}
\centering
\includegraphics[width=\linewidth]{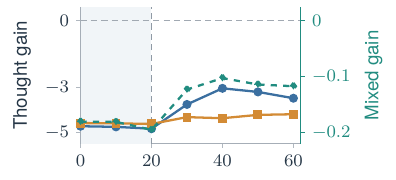}
\caption{Naive PPO gains (nats) over training steps.  Left: CSQA (blue) and GSM8K (orange) thought gains.  Right: CSQA mixed gain (teal dashed).  Shading: 20-step critic warmup with the actor frozen.}
\label{fig:learning-dynamics}
\end{wrapfigure}
The historical naive PPO run illustrates the motivation for a prediction-based warmup.  After 20 critic-only steps, actor updates improved CommonsenseQA thought gain from about $-4.8$ to $-3.0$ nats at step 40, and mixed gain from about $-0.18$ to $-0.10$ nats.  Both remained below the no-thought reference (Figure~\ref{fig:learning-dynamics}).  This run used neither \tfour{} nor the $R^2$ rule.  It provides a diagnostic of fixed-duration warmup, not an evaluation of the proposed method (Appendix~\ref{app:historical-naive-ppo}).

Qualification answers whether each critic predicts current returns.  It does not determine how two qualified advantage estimates should be combined, which is the separate problem addressed in Section~\ref{sec:twin-calibration}.
\finishwrap

\section{Learning Complementary Advantages from Single Rollouts}
\label{sec:twin-calibration}

Even after both critics qualify, the actual rollout law $q$ can differ from the recorded policy $\pi_{\mathrm{old}}$ used in the PPO ratio.  This training--inference mismatch, as well as clipping, can keep a common advantage offset from canceling in the policy update.  We combine the critics' advantages to reduce the offset's contribution while retaining an aggregate learning signal.

\subsection{Action-Dependent Mixing and Conditional Drift}
\label{sec:complementary-advantages}
With qualified critics in place, we examine how advantages enter clipped PPO, why a common offset can remain, and how the two critics provide room for calibration.

\paragraph{How advantages enter clipped PPO.}
\tfour{} retains $\mathcal J_{\mathrm{clip}}(\theta)$ from Eq.~\eqref{eq:ppo} and changes only its advantage signal.  With normalized advantage $\widetilde A$ detached, the per-token derivative away from the clipping thresholds is
\begin{equation}
\begin{aligned}
\nabla_\theta\min\{\rho\widetilde A,\clip(\rho,\rho_-,\rho_+)\widetilde A\}
&=\widetilde A\,\rho I\nabla_\theta\log\pi_\theta(a\mid s)=\widetilde A v,\\
I&=\one\{\widetilde A>0,\rho<\rho_+\}
+\one\{\widetilde A<0,\rho>\rho_-\}+\one\{\widetilde A=0\},
\end{aligned}
\label{eq:clipped-token-gradient}
\end{equation}
where $\one\{\cdot\}$ is the indicator and $v=\rho I\nabla_\theta\log\pi_\theta(a\mid s)$ is the masked update factor.  Averaging over replay gives $\nabla_\theta\mathcal J_{\mathrm{clip}}(\theta)=\E_{\mathcal D}[\widetilde A v]$.  The advantage supplies sign and strength, while $I$ turns off positive-advantage updates above $\rho_+$ and negative-advantage updates below $\rho_-$.  Thus PPO stops rewarding further movement once the probability has moved far enough in the favored direction.

\paragraph{Why a common offset matters.}
Separating each factor into its mean and variation reveals two components of the conditional update:
\begin{equation}
\begin{aligned}
\E_q[\widetilde A v\mid s]
&=\E_q\!\left[(\widetilde A-\E_q[\widetilde A\mid s])(v-\E_q[v\mid s])\mid s\right]
+\E_q[\widetilde A\mid s]\,\E_q[v\mid s]\\
&=\Cov_q(\widetilde A,v\mid s)+\textcolor{keyred}{\mu_{\widetilde A}(s)\,\mu_v(s)}
=\Cov_q(\widetilde A,v\mid s)+\Delta_\mu(s),
\end{aligned}
\label{eq:ppo-drift-decomposition}
\end{equation}
where $\mu_{\widetilde A}(s)=\E_q[\widetilde A\mid s]$ is the conditional advantage mean and $\mu_v(s)=\E_q[v\mid s]$ is the conditional mean update factor.  Figure~\ref{fig:advantage-calibration-insight}(a)--(b) contrasts the action-difference signal with the mean-induced component.  Under matched sampling without clipping, the score identity gives $\mu_v(s)=0$.  Mismatch or clipping can prevent that cancellation, allowing a common offset to contribute even when action rankings are unchanged.  \tfour{} targets this advantage offset and applies clipping after mixing.

\begin{figure}[!t]
\centering
\includegraphics[width=\linewidth]{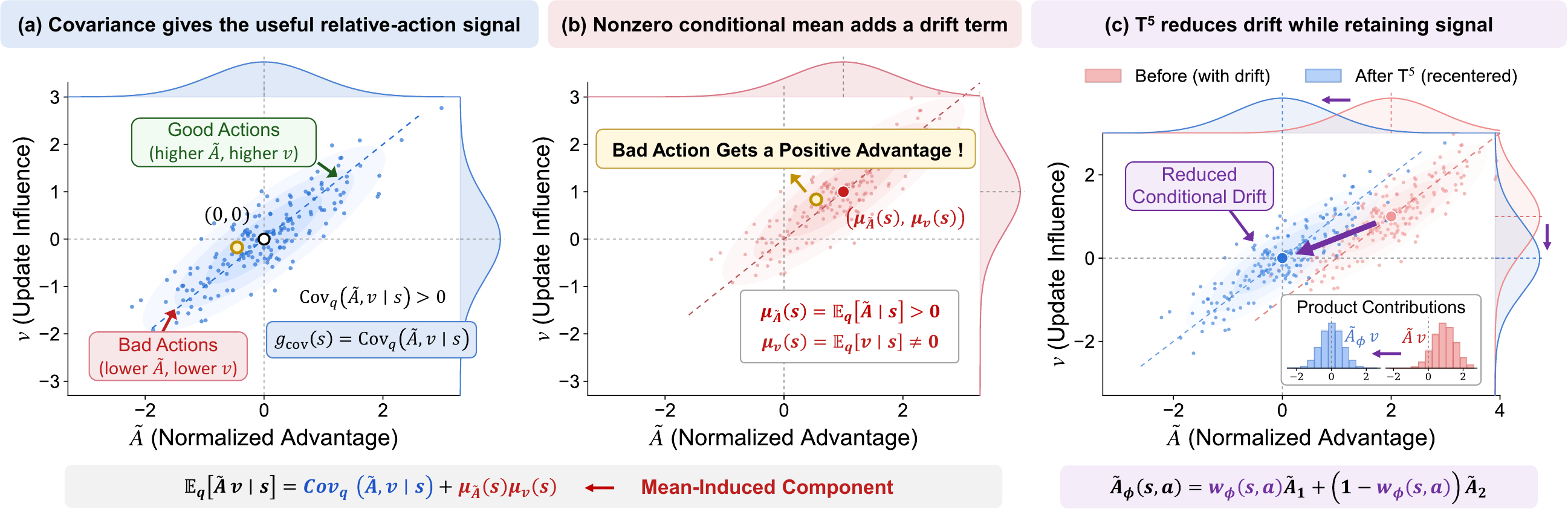}
\caption{Why calibrate twin-critic advantages? Schematic conditional distributions at a fixed prefix, showing normalized advantage against a scalar projection of the masked update factor $v$. (a) Centered advantages retain an action-contrast signal. (b) A common offset adds a mean-induced component when the mean update factor is nonzero. (c) Twin-advantage mixing targets the offset while retaining aggregate signal.}
\label{fig:advantage-calibration-insight}
\end{figure}

\paragraph{How twin critics enable calibration.}
Both critics evaluate the same trajectory.  Their raw GAE estimates $A_i$ use a common normalization, $\widetilde A_i=(A_i-\mu_{\mathrm{pilot}})/\sigma_{\mathrm{pilot}}$, with a shared mean and positive standard deviation fixed from independent pilot data (Appendix~\ref{app:advantage-snapshots}).  The mixed signal is
\begin{equation}
\widetilde A_\varphi=w_\varphi(s,a)\widetilde A_1+
(1-w_\varphi(s,a))\widetilde A_2,
\label{eq:twin-advantage}
\end{equation}
where $w_\varphi(s,a)\in[0,1]$ selects between the two estimates using the current action and frozen visible-prefix and critic features, not realized returns or GAEs.  The corresponding mask, update factor, and mean-induced component are denoted $I_\varphi$, $v_\varphi$, and $\Delta_{\mu,\varphi}$.

Figure~\ref{fig:advantage-calibration-insight}(a)--(b) is a distribution-level schematic, not a tracked-action counterexample.  In the concrete construction of Appendix~\ref{app:twin-disagreement}, the clipped-PPO gradient is $-0.105$ while the true-return gradient is $0.21$, showing that the mean-induced component can overwhelm a positive covariance signal.  Figure~\ref{fig:advantage-calibration-insight}(c) illustrates the repair: remove a common offset without erasing the signal.  The following two-action example is a separate numerical construction.  For two equally likely actions, a fixed rule $w$ can change the twin average $\overline A$ into
\begin{equation}
\overline A=(0.6,-0.4)
\quad\longrightarrow\quad
\widetilde A_w=(0.6,-0.6),
\label{eq:complementary-toy}
\end{equation}
where the mean falls from $0.1$ to zero while the action contrast stays nonzero.  The signs are unchanged, so the PPO mask and $\mu_v(s)$ remain the same at fixed policy ratios.  The mean-induced term therefore vanishes through advantage centering alone.  Appendix~\ref{app:twin-mixture-example} gives the original estimates and weights.  Section~\ref{sec:conditional-calibration} specifies an objective for selecting the mixture from single rollouts.

\subsection{A Single-Rollout Saddle-Point Objective}
\label{sec:conditional-calibration}
To select the mixture illustrated above from single rollouts, we answer three questions: what objective do we use, why does it work, and what does optimal mixing mean?

\paragraph{What objective do we use?}
We seek a small squared conditional offset $R(\varphi)=\E_d[\mu_{\widetilde A_\varphi}(s)^2]$, not a small advantage at every token.  Directly minimizing advantage squares would also suppress within-context variation, including action differences useful to PPO (Appendix~\ref{app:moment-duality}).  Alongside $w_\varphi(s,a)$, introduce an auxiliary mean function $h(c)$ from a class $\mathcal H$.  It may use the full pre-action context, including scoring text, but not the sampled action or continuation.  With the critics and normalization fixed, we consider the constrained saddle objective
\begin{equation}
\boxed{\min_\varphi\sup_{h\in\mathcal H}
\underbrace{\E_{d,q}[2h(c)\widetilde A_\varphi-h(c)^2]}_{\mathcal L(\varphi,h)}
\quad\text{subject to}\quad C(\varphi)\leq\kappa Q},
\label{eq:calibrated-weight}
\end{equation}
where $C(\varphi)=\E_{d,q}[(\widetilde A_\varphi-\overline A)^2]$ measures displacement from the fixed average $\overline A=(\widetilde A_1+\widetilde A_2)/2$, $Q=\E_{d,q}[\overline A^2]$ is its squared size, and $0\leq\kappa<1$ limits their ratio.  The mixture minimizes this objective, while the auxiliary function maximizes it.  We write $h_\zeta$ for a selected auxiliary response, with parameters $\zeta$.  Both advantages and the objective's per-token quantity come from one trajectory per selected position, without repeated prefix continuations.  Parameterization and optimization details are given in Appendix~\ref{app:calibration-implementation}.

\paragraph{Why does this objective work?}
For a fixed mixture, completing the square identifies the best mean response.  Over all functions with finite $\E_d[h^2]$, denoted $L^2(d)$,
\begin{equation}
R(\varphi)=\sup_{h\in L^2(d)}\mathcal L(\varphi,h),
\qquad h_\varphi^*(c)=\mu_{\widetilde A_\varphi}(s).
\label{eq:conditional-moment-duality}
\end{equation}
The maximizing mean function exposes the remaining offset, while the minimizing weight changes the mixture to reduce it.  Minimizing over both arguments would instead reward an inaccurate mean response.  For a restricted class $\mathcal H$, the unresolved mean error enters the bound in Section~\ref{sec:conditional-drift}.  The full square-completion argument is in Appendix~\ref{app:moment-duality}.
The constraint protects the signal during this correction.  The fixed average $w=1/2$ is feasible, and for $Q>0$ every feasible mixture satisfies
\begin{equation}
\|\widetilde A_\varphi\|_{L^2}\geq(1-\sqrt\kappa)\|\overline A\|_{L^2}>0,
\label{eq:main-signal-retention}
\end{equation}
where $\|X\|_{L^2}=(\E_{d,q}[X^2])^{1/2}$.  Feasible mixtures lie in a ball around the fixed average that excludes zero.  Thus a smaller conditional mean need not come from erasing the aggregate learning signal.  The reverse triangle inequality proves this in one step (Appendix~\ref{app:signal-retention}).

\paragraph{What does optimal mixing mean?}
Let $x=(s,a)$ denote only the gate's visible inputs and $\Delta\widetilde A=\widetilde A_1-\widetilde A_2$.  The population optimum has the following form.
\begin{proposition}[Optimal mixing]
\label{prop:population-saddle}
Fix a finite nonzero rollout occupation measure and square-integrable advantages, with $Q>0$ and $0<\kappa<1$.  Over all measurable gates $w(x)\in[0,1]$ and $h\in L^2(d)$, the constrained problem has a Lagrangian saddle point $(w^*,h^*,\beta^*)$, where $h^*(c)=\E_q[\widetilde A_{w^*}\mid c]$ and $\beta^*\geq0$ is the retention multiplier.  Wherever $\beta^*>0$ and $\E_{d,q}[(\Delta\widetilde A)^2\mid x]>0$,
\begin{equation}
w^*(x)=\clip\!\left(\frac12-
\frac{\E_{d,q}[h^*(c)\Delta\widetilde A\mid x]}
{\beta^*\E_{d,q}[(\Delta\widetilde A)^2\mid x]},\,0,1\right).
\label{eq:optimal-saddle-weight}
\end{equation}
\end{proposition}
The numerator sets the adjustment direction through the relation between offset and critic disagreement.  The denominator scales it by disagreement and retention pressure, and projection keeps the result between the two estimates.  This self-consistent relation characterizes the population optimum, not an extra prefix-wise estimation step.  Strong duality, the precise function-space conditions, and boundary cases are given in Appendix~\ref{app:population-saddle}.

Algorithm~\ref{alg:t5-training} places this objective within PPO.  The selected mixture remains fixed during actor updates, and the complete mixed advantage is detached.  Appendix~\ref{app:calibration-implementation} describes the parameterized realization and its fitting and validation procedure.

\begin{algorithm}[H]
\caption{\tfour{}: one training window}
\label{alg:t5-training}
\small
\begin{algorithmic}[1]
\REQUIRE Actor $\pi_\theta$, critics $V_1,V_2$, displacement limit $\kappa$.
\STATE Warm up the value heads, then the full critics using $\mathcal L_{V_i}$ (Eq.~\eqref{eq:value-loss}).\hfill\textcolor{black!55}{\textit{// Critic fitting}}
\STATE Require $m$ fresh tests with $\min(R_1^2,R_2^2)\geq\eta$ (Eq.~\eqref{eq:holdout-r2}). If not, pause the actor and refit.
\STATE Learn $w_\varphi,h_\zeta$ (Eq.~\eqref{eq:calibrated-weight}).\hfill\textcolor{black!55}{\textit{// Advantage calibration}}
\STATE Validate and freeze $w_\varphi$ (Appendix~\ref{app:moment-validation}).
\STATE Sample fresh trajectories and form $\widetilde A_\varphi$ (Eq.~\eqref{eq:twin-advantage}).\hfill\textcolor{black!55}{\textit{// Policy optimization}}
\STATE Update $\theta$ (Eq.~\eqref{eq:ppo}).
\end{algorithmic}
\end{algorithm}
Each selected position supplies one thought to both critics.  Additional critic-fitting and mixture-validation trajectories are included in the group-sampling cost comparison in Appendix~\ref{app:rollout-accounting}.

\subsection{Bounding Conditional Drift}
\label{sec:conditional-drift}
For fitted heads, two quantities connect the learned objective to residual drift.  The conditional-mean prediction error is $\epsilon_h:=\E_d[(h_\zeta(c)-\mu_{\widetilde A_\varphi}(s))^2]$.  An independent validation estimate $\widehat{\mathcal L}_{\mathrm{val}}$ measures the objective, with uncertainty $\epsilon_{\mathrm{stat}}$.  The first quantity concerns the true conditional mean, not regression error against individual sampled advantages.  Its approximation and optimization components are detailed in Appendix~\ref{app:moment-duality}.

\begin{theorem}
\label{thm:complementary-drift}
Condition on the frozen rollout law, critics, scorer, normalizer, and selected head parameters.  Assume square-integrable advantages and auxiliary functions, valid recorded probabilities, differentiable policies, and the required finite score moments.  Let $B$ bound the weighted size of the mean update factor, $\E_d\|\mu_{v_\varphi}\|_2^2\leq B^2$, where $\|\cdot\|_2$ is the Euclidean norm.  For a frozen candidate and an independent validation estimate with
$\mathcal L(\varphi,h_\zeta)\leq\widehat{\mathcal L}_{\mathrm{val}}+\epsilon_{\mathrm{stat}}$,
\begin{equation}
\E_d\|\Delta_{\mu,\varphi}\|_2\leq B\sqrt{R(\varphi)}
\leq B\sqrt{\max\{0,\widehat{\mathcal L}_{\mathrm{val}}+
\epsilon_{\mathrm{stat}}+\epsilon_h\}}.
\label{eq:main-drift-budget}
\end{equation}
\end{theorem}
The bound links the remaining advantage offset to its amplification by the policy update.  Cauchy--Schwarz gives the first inequality, and the dual identity with validation gives the second (Appendix~\ref{app:drift-proof}).  A small measured objective gives a tight bound when the mean-prediction error and validation uncertainty are also small.  This controls the mean-induced component with the actual PPO mask, rather than the entire update.  Appendix~\ref{app:moment-validation} specifies the validation conditions.

\section{Experiments}
\label{sec:experimental-configuration}

\iffalse
We organize the experiments around three questions:
\begin{enumerate}
    \item Does \tfour{} improve downstream performance over representative strong baselines?
    \item Are the gains consistent across tasks and evaluation settings?
    \item How does end-to-end training efficiency scale with model size?
\end{enumerate}
\fi

We organize the experiments around three questions:
\begin{enumerate}
    \item Does \tfour{} improve downstream performance over matched controls and strong single-critic and group-relative baselines?
    \item How consistently do the gains hold across task families, evaluation settings, and model scales?
    \item What do the mechanism diagnostics reveal, and how do performance--time trade-offs and per-step costs compare with strong baselines?
\end{enumerate}

\subsection{Experimental Setup}

We continue pretraining Qwen3.5-2B-Base and Qwen3.5-9B-Base on a fixed corpus with disjoint data splits \citep{qwen2026qwen35}.  The 2B study compares Base, NTP, single-critic Quiet-STaR (w/ PPO), Quiet-STaR (w/ GRPO), RLP, and \tfour{} under matched settings \citep{zelikman2024quietstar,hatamizadeh2025rlp}, while 9B measures scaling.  All GRPO-based baselines use group size 8.  Appendix~\ref{app:experimental-details} gives the full protocol, timing traces, and compute accounting.

\begin{figure}[!t]
\centering
\renewcommand{\thesubfigure}{(\alph{subfigure})}
\captionsetup[subfigure]{labelformat=simple,skip=2pt}
\begin{subfigure}[t]{0.575\linewidth}
\centering
\includegraphics[width=\linewidth]{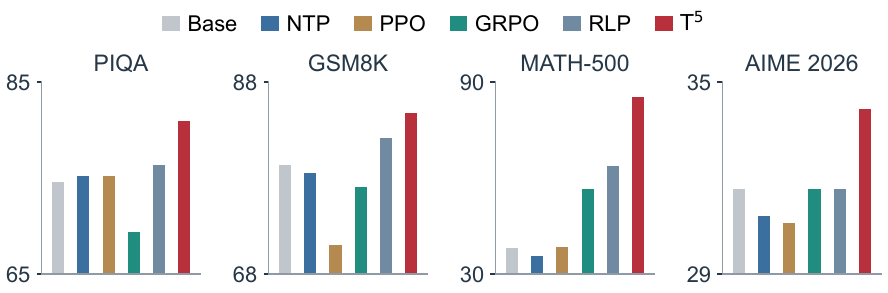}
\caption{Benchmark results}
\label{fig:qwen35-2b-benchmark-comparison}
\end{subfigure}\hfill
\begin{subfigure}[t]{0.405\linewidth}
\centering
\raisebox{3pt}{\includegraphics[width=\linewidth]{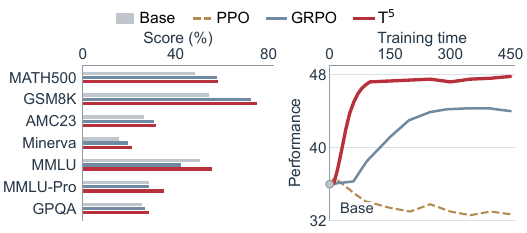}}
\caption{Controlled comparison}
\label{fig:benchmark-comparison}
\label{fig:training-time}
\label{fig:training-by-step}
\end{subfigure}
\caption{Experimental comparisons. (a) Selected Qwen3.5-2B results (PIQA: accuracy; others: pass@1). (b) Earlier controlled comparison with a shared Quiet-STaR architecture: benchmark scores (\%, left) and performance over training time from the common Base checkpoint (right). Legend labels PPO and GRPO denote Quiet-STaR (w/ PPO) and Quiet-STaR (w/ GRPO), respectively.}
\label{fig:benchmark-and-training-time}
\end{figure}

\subsection{Main Results}

\paragraph{Overall downstream effectiveness.}
\label{sec:qwen35-2b-results}
Table~\ref{tab:qwen35-2b-main-results} shows that \tfour{} leads all seven retained benchmarks.  Figure~\ref{fig:qwen35-2b-benchmark-comparison} highlights four with the clearest separation.  Its unweighted mean is 69.56, versus 64.52 for RLP and 60.15 for Quiet-STaR (w/ GRPO).  The largest RLP margin is on MATH-500 (85.45 versus 63.65).  Gains on the other six tasks are 0.90--4.51 percentage points.

After correcting its AIME 2025/2026 scores to 33.00/30.83, NTP averages 56.30 and is nearly unchanged from Base (56.28).  Quiet-STaR (w/ PPO) averages 54.42, below Base in this run.  Without seed-level uncertainty, these point estimates establish only the observed ordering, not statistical significance or component-level causality.  Appendix~\ref{app:qwen35-2b-results} gives the rounded table, Pass@4, reasoning-token diagnostics, and provenance.

\paragraph{Cross-task consistency.}
\label{sec:preliminary-benchmarks}
Before the full target protocol, we use an earlier controlled comparison to ask whether the advantage persists across mathematical and knowledge-intensive tasks \citep{yang2025qwen3}.  All trained variants share the Quiet-STaR hidden-thought architecture.  The two policy-optimization baselines are therefore denoted Quiet-STaR (w/ PPO) and Quiet-STaR (w/ GRPO) \citep{zelikman2024quietstar,shao2024deepseekmath}.

Figure~\ref{fig:benchmark-comparison} compares Base, Quiet-STaR (w/ GRPO), and \tfour{} on seven mathematical and knowledge-intensive tasks.  \tfour{} exceeds Quiet-STaR (w/ GRPO) on all seven, with a small difference on MATH500 and larger differences on GSM8K and MMLU.  It also exceeds Base on the displayed tasks, while Quiet-STaR (w/ GRPO) improves mathematics but falls 8.21 points below Base on MMLU.  The full summary includes an exception: \tfour{} is slightly below Base on GPQA pass@1.  Appendix~\ref{app:preliminary-provenance} gives task references, all scores, and source information.  These point estimates motivate controlled evaluation but do not isolate predictive qualification or complementary weighting.

\begin{wraptable}{r}{0.35\textwidth}
\centering
\caption{Mean step time (s).}
\label{tab:qwen35-mean-step-time}
\begin{tabular}{@{}lcc@{}}
\toprule
\textbf{Model} & \textbf{RLP} & $\mathbf{T}^{5}$ \\
\midrule
Qwen3.5-2B & 54.01 & 31.79 \\
Qwen3.5-9B & 497.83 & 182.18 \\
\bottomrule
\end{tabular}
\end{wraptable}
\paragraph{Training-time trade-off.}
\label{sec:training-dynamics}
Group sampling, value fitting, and calibration incur different costs.  Against RLP, Table~\ref{tab:qwen35-mean-step-time} shows mean step-time reductions of 41.1\% at 2B and 63.4\% at 9B, with cross-scale performance in Appendix~\ref{app:qwen35-scale-results}.  Separately, Figure~\ref{fig:training-time} traces the earlier Quiet-STaR (w/ PPO), Quiet-STaR (w/ GRPO), and \tfour{} runs from their common Base checkpoint.  Timings use the same hardware and include critic warmup, thought generation, reward scoring, critic fitting, weight and mean learning, validation, and rejected windows.  The relevant question is whether maintaining useful actor signals compensates for the cost of two critics and auxiliary continuations.  Appendix~\ref{app:experimental-details} gives the full reporting protocol.
\finishwrap

\section{Conclusion}

\ifshortconclusion
\tfour{} separates return prediction from advantage calibration.  Qualified twin critics supply estimates that a conditional-moment saddle objective combines using one rollout per selected position.  Its population optimality conditions explain how the remaining offset, critic disagreement, and retention constraint determine the mixture.  The retention constraint preserves aggregate signal strength, and the analysis bounds the mean-induced component of clipped PPO drift.  This makes both return prediction and the use of its advantage estimates part of the learning process.

\else
Reliable actor--critic learning requires both useful return predictions and a suitable way to turn them into policy updates.  \tfour{} addresses these requirements through predictive qualification and learned advantage mixing.  Two critics evaluate the same trajectory, providing a family of convex corrections.  An auxiliary function learns the mixed conditional mean across contexts and guides the weight, while a displacement constraint preserves aggregate signal strength.

The conditional-moment formulation supports stochastic learning from one trajectory per selected input and yields a bound on mean-induced PPO drift.  Its quality depends on the approximation and optimization of the mean predictor, as well as finite-data uncertainty.  These conditions clarify what advantage calibration can control and why it complements return-prediction accuracy.  The preliminary benchmark summary motivates further evaluation, while the specified protocol separates mechanism attribution from total training cost.  Establishing practical gains requires measuring whether the improved actor signal compensates for the two critics and auxiliary data collection.

\fi

\Needspace{12\baselineskip}
\section*{Reproducibility Statement}
The method and training objectives are specified in the main text, with derivations in Appendices~\ref{app:value-gradient-proof} and~\ref{app:conditional-calibration}.  Appendix~\ref{app:data-contracts} details data separation, parameter updates, and trajectory counts.  Appendix~\ref{app:experimental-details} documents the evaluation protocol and the provenance of the preliminary scores and historical diagnostic.  The coordinates underlying the training-time figure are retained with the plotting data.

\Needspace{8\baselineskip}
\bibliographystyle{arxiv_preprint}
\bibliography{t5}

\clearpage
\appendix

\begingroup
\renewcommand{\contentsname}{Full Paper Contents}
\setcounter{tocdepth}{2}
\tableofcontents
\endgroup
\clearpage

\newpage
\section{Additional Experimental Details and Results}
\label{app:experimental-details}
This appendix first gives the model, data, and evaluation protocols, then follows the three questions in Section~\ref{sec:experimental-configuration}: overall performance, cross-task consistency, and training efficiency across model scales.  Historical training diagnostics follow these results.

\subsection{Model, Data, and Implementation Configuration}
\label{app:fixed-configuration}

The target protocol uses the two 2026 Qwen3.5 Base checkpoints in Table~\ref{tab:base-models} \citep{qwen2026qwen35}.\footnote{Model resources: \url{https://huggingface.co/Qwen/Qwen3.5-2B-Base} and \url{https://huggingface.co/Qwen/Qwen3.5-9B-Base}.  These identifiers specify the selected initializations.}  Approximate file sizes describe checkpoint storage, not training memory.  Neither initialization is an Instruct checkpoint.

\begin{table}[htbp]
\centering
\small
\renewcommand{\arraystretch}{1.15}
\caption{Target Base-model initializations.}
\label{tab:base-models}
\begin{tabular*}{\linewidth}{@{\extracolsep{\fill}}lrr@{}}
\toprule
\rowcolor{tableaccent}\textbf{Checkpoint} & \textbf{Parameters} & \textbf{Storage} \\
\midrule
\texttt{Qwen/Qwen3.5-2B-Base} & 2.274B & 4.56\,GB \\
\texttt{Qwen/Qwen3.5-9B-Base} & 9.653B & 19.32\,GB \\
\bottomrule
\end{tabular*}
\end{table}

The corpus release is \texttt{recent\_fast\_100step\_v1}.  Table~\ref{tab:training-mixture} specifies the source mixture, and Table~\ref{tab:training-splits} preserves the supplied sequence and corpus-token counts.  FineWeb supplies web text \citep{penedo2024fineweb},\footnote{FineWeb dataset: \url{https://huggingface.co/datasets/HuggingFaceFW/fineweb}.  The selected crawl identifiers are listed in Table~\ref{tab:training-mixture}.} arXiv supplies 2025--2026 titles and abstracts, and OpenWebMath supplies mathematical web text \citep{paster2023openwebmath}.  The supplied configuration calls the scholarly component \texttt{open-index/open-arxiv}.  Source dates alone cannot rule out benchmark contamination.

\begin{table}[htbp]
\centering
\small
\renewcommand{\arraystretch}{1.15}
\caption{Training-corpus composition.  The four web crawls are from 2025.}
\label{tab:training-mixture}
\begin{tabular*}{\linewidth}{@{\extracolsep{\fill}}lr@{}}
\toprule
Source & Share \\
\midrule
FineWeb: \texttt{CC-MAIN-2025-13/18/21/26} & 85\% \\
2025--2026 open-index/open-arxiv titles and abstracts & 10\% \\
OpenWebMath & 5\% \\
\bottomrule
\end{tabular*}
\end{table}

\begin{table}[htbp]
\centering
\small
\renewcommand{\arraystretch}{1.15}
\caption{Corpus splits, with 512 tokens per sequence.  Counts exclude generated thoughts and auxiliary continuations.}
\label{tab:training-splits}
\begin{tabular*}{\linewidth}{@{\extracolsep{\fill}}lrr@{}}
\toprule
\rowcolor{tableaccent}\textbf{Split} & \textbf{Sequences} & \textbf{Corpus tokens} \\
\midrule
Training & 12,800 & 6,553,600 \\
Critic calibration & 1,024 & 524,288 \\
Development & 1,024 & 524,288 \\
\bottomrule
\end{tabular*}
\end{table}

Actor optimization and critic return fitting use the training split.  Predictive qualification uses the calibration split, with held-out outcomes excluded from fitting and from selecting a threshold on those outcomes.  Development data support progress measurement and a checkpoint-selection rule fixed before final evaluation.  Complete text items or parent trajectories are partitioned into the roles in Appendix~\ref{app:data-contracts} before token batching.  Head learning, independent validation, and actor optimization use separate ordinary trajectories, each with one thought per selected item.  Predictive qualification alone does not certify the learned mixture's conditional moment risk.

\Needspace{24\baselineskip}
\begin{table}[htbp]
\centering
\small
\renewcommand{\arraystretch}{1.18}
\setlength{\tabcolsep}{3pt}
\caption{Methods in the reported Qwen3.5-2B comparison.  A dash indicates no grouped thought sampling.  All GRPO-based comparisons in this paper use group size 8.}
\label{tab:target-baselines}
\begin{tabular}{@{}>{\raggedright\arraybackslash}p{0.22\linewidth}c>{\raggedright\arraybackslash}p{0.22\linewidth}>{\raggedright\arraybackslash}p{0.36\linewidth}@{}}
\toprule
\rowcolor{tableaccent}\textbf{Method} & \textbf{Group} & \textbf{Update} & \textbf{Controlled role} \\
\midrule
Base & -- & None & Original checkpoint with no additional corpus exposure. \\
NTP & -- & Next-token likelihood & Controls for continued exposure to the same training mixture. \\
Quiet-STaR (w/ PPO) & 1 & Single-critic clipped PPO & Uses the hidden-thought interface with a standard learned-value update. \\
RLP & 8 & GRPO-style group-relative update & Samples eight thoughts and uses next-token information-gain rewards. \\
Quiet-STaR (w/ GRPO) & 8 & Group-relative update & Applies GRPO to the same hidden-thought architecture. \\
\rowcolor{tableaccent}\tfour{} & 1 & Twin-critic clipped PPO & Uses predictive qualification and constrained saddle-point calibration. \\
\bottomrule
\end{tabular}
\end{table}

\paragraph{Execution environment.}
The target execution environment uses three nodes with NVIDIA H20 GPUs, CUDA 12.9, and Python 3.13.

\paragraph{Calibration parameterization and optimization.}
\label{app:calibration-implementation}
The functions in Section~\ref{sec:conditional-calibration} can be parameterized by lightweight MLP heads rather than additional language-model backbones.  The weight head $w_\varphi(s,a)$ has output in $[0,1]$ and uses frozen visible-prefix and critic features together with the current action.  The scalar mean head $h_\zeta(c)$ can additionally use the scoring context through a separate training-only input path, but not the sampled action or its continuation.  Returns and GAEs are labels, not inputs.  MLP depth, width, activations, feature encodings, optimizer settings, and update ratios are run-specific choices to record for the implementation, rather than requirements of the objective.

During calibration, fix the rollout law, critics, scorer, and common pilot normalizer, and detach the GAE labels and critic features.  Minimize in $\varphi$ and maximize in $\zeta$ and the nonnegative retention multiplier, using the single-trajectory gradients in Appendix~\ref{app:single-rollout-gradients}.  Each selected item contributes one thought, shared by both critics.  Freeze the candidate before independent validation, then construct detached advantages on fresh actor trajectories.  A failed retention check uses the fixed average, and failed critic qualification pauses actor updates.  Appendix~\ref{app:restart-contract} specifies the trajectory-level separation and replay rules.  This parameterized fitting procedure realizes the saddle objective without changing the actor's clipped PPO loss.

\subsection{Target Evaluation and Comparison Protocol}
\label{app:evaluation-protocol}

Table~\ref{tab:evaluation-suite} defines the seven-benchmark primary suite.  AIME 2025 and 2026 use the official competition sets \citep{aime}.  GSM8K \citep{cobbe2021verifiers} and MATH-500 \citep{hendrycks2021math} cover mathematical reasoning.  CommonsenseQA \citep{talmor2019commonsenseqa}, PIQA \citep{bisk2020piqa}, and ARC-Challenge \citep{clark2018arc} cover general, physical, and scientific commonsense.

\begin{table}[htbp]
\centering
\small
\renewcommand{\arraystretch}{1.16}
\setlength{\tabcolsep}{4pt}
\caption{Primary evaluation suite and frozen item counts.}
\label{tab:evaluation-suite}
\begin{tabular*}{\linewidth}{@{\extracolsep{\fill}}llrl@{}}
\toprule
\rowcolor{tableaccent}\textbf{Release} & \textbf{Benchmark} & \textbf{Items} & \textbf{Focus} \\
\midrule
2025 & AIME 2025 & 30 & Competition mathematics \\
2026 & AIME 2026 & 30 & Competition mathematics \\
2021 & GSM8K test & 1,319 & Grade-school mathematics \\
2021 & MATH-500 & 500 & Competition mathematics \\
2019 & CommonsenseQA validation & 1,221 & General commonsense \\
2020 & PIQA validation & 1,838 & Physical commonsense \\
2018 & ARC-Challenge test & 1,172 & Science reasoning \\
\bottomrule
\end{tabular*}
\end{table}

\iffalse
The earlier protocol also listed HMMT, SimpleQA Verified, LiveBench Reasoning, and HLE.  We retain them as supplementary diagnostics rather than mixing them into the primary aggregate.  Table~\ref{tab:supplementary-evaluation} preserves their intended inventories and access requirements.\footnote{Resource locations: \url{https://www.hmmt.org/archive}, \url{https://huggingface.co/datasets/google/simpleqa-verified}, \url{https://livebench.ai/}, and \url{https://lastexam.ai/}.  Links do not replace the frozen manifests and scorer versions recorded for a run.}

\begin{table}[htbp]
\centering
\small
\renewcommand{\arraystretch}{1.15}
\setlength{\tabcolsep}{4pt}
\caption{Retained supplementary diagnostic suite.  These tasks are reported separately from the primary aggregate.}
\label{tab:supplementary-evaluation}
\begin{tabular*}{\linewidth}{@{\extracolsep{\fill}}llrl@{}}
\toprule
\rowcolor{tableaccent}\textbf{Release} & \textbf{Benchmark} & \textbf{Items} & \textbf{Focus} \\
\midrule
2025 & HMMT February 2025 & 30 & Competition mathematics \\
2025 & SimpleQA Verified & 1,000 & Factual accuracy \\
Through 2025 & LiveBench Reasoning & $\sim$200 & General reasoning \\
2025 & Humanity's Last Exam & $\sim$2,500 & Cross-domain reasoning \\
\bottomrule
\end{tabular*}
\end{table}
\fi

AIME, GSM8K, and MATH-500 use fixed answer extraction and normalization before exact-match scoring.  CommonsenseQA, PIQA, and ARC-Challenge use accuracy on their frozen multiple-choice splits.  All primary benchmarks are excluded from training, critic fitting, admission, calibration, and checkpoint selection, and score interpretation requires corpus-overlap checks.

Within each model scale, methods share the same Base initialization, corpus, selected positions, evaluation prompts, demonstrations, output-token limit, decoding parameters, and tool or retrieval allowances.  Base checkpoints use an explicit prompting protocol rather than an assumed Instruct template.  The 512-token training sequence does not determine evaluation generation length.  The reported results are point estimates without variation across independent training seeds, which is especially limiting for the 30-item AIME sets.

The comparison roles are summarized in Table~\ref{tab:target-baselines}.  NTP controls corpus exposure, while Quiet-STaR (w/ PPO) and Quiet-STaR (w/ GRPO) compare two policy-update choices on the hidden-thought architecture.  Group-8 RLP provides the closest information-gain baseline \citep{hatamizadeh2025rlp}.  Comparisons distinguish matched corpus tokens, sampled thought tokens, and total compute.  Compute accounting follows Eq.~\eqref{eq:total-rollout-budget} and includes trajectory-set sizes, generated trajectories and tokens, actor-eligible tokens, group sizes, and head-update counts.  Head-training and validation trajectories count toward generation without contributing actor tokens.  Total cost includes selected positions, thought length, candidate generation, reward scoring, both critics, head-learning and validation sets, rejected windows, and hyperparameter search.  Equal optimizer steps do not imply equal compute.

Development diagnostics include continuation loss, thought and mixed NLL gain, useful-position fraction, and worst-1\% gain.  Each critic is evaluated by held-out $R^2$, prediction standard deviation, and twin disagreement.  Policy KL, clipping frequency, actor-release timing, the restricted moment objective, signal-retention ratio, and head/validation cost characterize the new protocol.  Reproducibility additionally requires optimizer settings, batch and accumulation sizes, precision, hardware, trainable parameters, seeds, thought-length limits, head architectures and inputs, $\kappa$, and split rules.

\subsection{Qwen3.5-2B Results and Diagnostics}
\label{app:qwen35-2b-results}

\paragraph{Primary benchmark results.}
The supplied Qwen3.5-2B result matrix contains the six methods in Table~\ref{tab:target-baselines}: Base, NTP, Quiet-STaR (w/ PPO), RLP, Quiet-STaR (w/ GRPO), and \tfour{}.  Table~\ref{tab:qwen35-2b-main-results} reports accuracy for CommonsenseQA, PIQA, and ARC-Challenge and pass@1 for the four mathematical tasks.  The values are author-supplied point estimates.

\begin{table}[htbp]
\centering
\small
\renewcommand{\arraystretch}{1.16}
\setlength{\tabcolsep}{2.1pt}
\caption{Qwen3.5-2B primary results on the seven retained benchmarks, rounded to integer percentages.  Bold and underlined entries identify the best and second-best unrounded values in each row.  The final row rounds the unweighted mean of the seven source values.}
\label{tab:qwen35-2b-main-results}
\begin{tabular*}{\linewidth}{@{\extracolsep{\fill}}lrrrrrr@{}}
\toprule
\rowcolor{tableaccent}\textbf{Benchmark} & \textbf{Base} & \textbf{NTP} & \textbf{\shortstack{Quiet-STaR\\(w/ PPO)}} & \textbf{RLP} & \textbf{\shortstack{Quiet-STaR\\(w/ GRPO)}} & \textbf{\tfour{}} \\
\midrule
CSQA & 66 & 67 & 64 & \underline{74} & 67 & \textbf{75} \\
PIQA & 75 & 75 & 75 & \underline{76} & 69 & \textbf{81} \\
ARC-Challenge & 74 & 74 & 75 & \underline{88} & 86 & \textbf{90} \\
GSM8K & 79 & 78 & 71 & \underline{82} & 77 & \textbf{85} \\
MATH-500 & 38 & 36 & 39 & \underline{64} & 56 & \textbf{85} \\
AIME 2025 & 30 & 33 & 27 & \underline{36} & 33 & \textbf{38} \\
AIME 2026 & 32 & 31 & 31 & \underline{32} & \underline{32} & \textbf{34} \\
\midrule
Retained-7 mean & 56 & 56 & 54 & \underline{65} & 60 & \textbf{70} \\
\bottomrule
\end{tabular*}
\end{table}

The current means are 56.28 for Base, 56.30 for NTP, 54.42 for Quiet-STaR (w/ PPO), 64.52 for RLP, 60.15 for Quiet-STaR (w/ GRPO), and 69.56 for \tfour{}.  \tfour{} leads on all seven tasks and exceeds RLP by 5.04 points on the retained-seven mean.  RLP exceeds Quiet-STaR (w/ GRPO) on six tasks and ties it on AIME 2026.  Figure~\ref{fig:qwen35-2b-benchmark-comparison} in the main text displays the unrounded values for four selected tasks.

\emph{Aggregation note.}  The retained-seven mean is recomputed from the unrounded source values in Table~\ref{tab:qwen35-2b-main-results}.  Earlier aggregate values of 55.95 for RLP and 51.57 for Quiet-STaR (w/ GRPO) covered a broader suite that included benchmarks later removed as obsolete, so they are not mixed with the current seven-task summary.  The NTP scores on AIME 2025/2026 are corrected to 33.00/30.83.

\paragraph{Sampling gains and reasoning length.}
Table~\ref{tab:qwen35-2b-t5-diagnostics} records the additional \tfour{} diagnostics.  Coverage is 100\% on every benchmark.  Pass@4 is available for the four generative mathematics evaluations and exceeds pass@1 by 15.83 points on each AIME set, 8.38 points on GSM8K, and 4.15 points on MATH-500.  The average reasoning length varies substantially by task, from 548.2 tokens on PIQA to 3922.4 on AIME 2026.  Because these token measurements are available only for \tfour{}, they describe its cross-task inference profile rather than a cross-method efficiency comparison.

\begin{table}[htbp]
\centering
\small
\renewcommand{\arraystretch}{1.12}
\setlength{\tabcolsep}{4pt}
\caption{Qwen3.5-2B \tfour{} evaluation diagnostics.  Primary scores, pass@4, and coverage are percentages.  A dash marks tasks for which pass@4 was not supplied.}
\label{tab:qwen35-2b-t5-diagnostics}
\begin{tabular*}{\linewidth}{@{\extracolsep{\fill}}llrrrr@{}}
\toprule
\rowcolor{tableaccent}\textbf{Benchmark} & \textbf{Metric} & \textbf{Primary} & \textbf{Pass@4} & \textbf{Coverage} & \textbf{Avg. tokens} \\
\midrule
AIME 2025 & pass@1 & 37.50 & 53.33 & 100.00 & 3849.1 \\
AIME 2026 & pass@1 & 34.17 & 50.00 & 100.00 & 3922.4 \\
ARC-Challenge & Accuracy & 89.59 & -- & 100.00 & 609.2 \\
CSQA & Accuracy & 74.53 & -- & 100.00 & 649.2 \\
GSM8K & pass@1 & 84.80 & 93.18 & 100.00 & 645.7 \\
MATH-500 & pass@1 & 85.45 & 89.60 & 100.00 & 1673.6 \\
PIQA & Accuracy & 80.90 & -- & 100.00 & 548.2 \\
\bottomrule
\end{tabular*}
\end{table}

\begin{figure}[htbp]
\centering
\includegraphics[width=\linewidth]{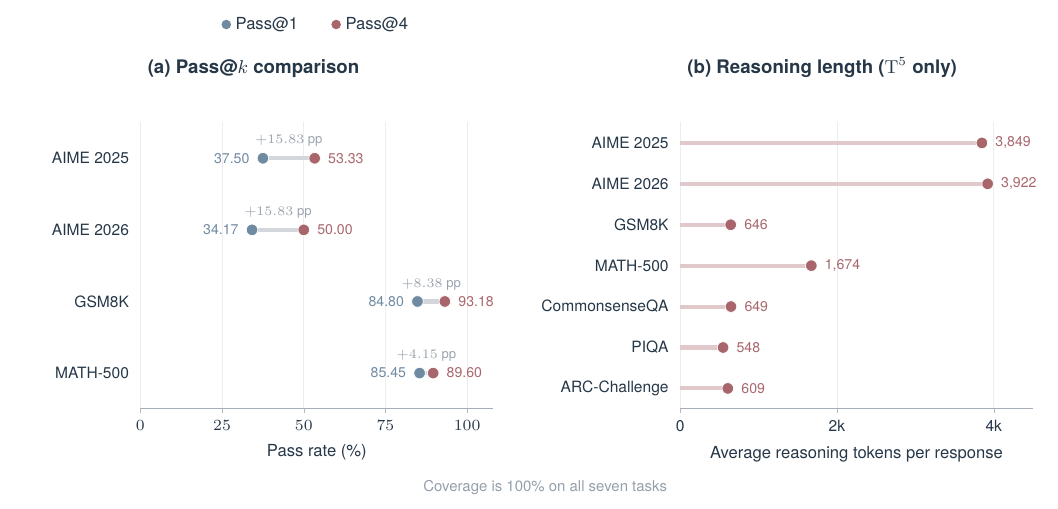}
\caption{Qwen3.5-2B \tfour{} diagnostics.  (a) Paired pass@1 and pass@4 with percentage-point differences.  (b) Average reasoning tokens per response, with rounded endpoint labels.  Coverage is 100\% on all seven tasks.}
\label{fig:t5-pass-and-inference-tokens}
\end{figure}

\Needspace{22\baselineskip}
\paragraph{Warmup parameter ablation.}
\label{app:warmup-ablation}
We vary the warmup parameter, denoted by $\tau$ in this sweep, over $\{0.00,0.05,0.10,0.20,0.40\}$ on Qwen3.5-2B and evaluate AIME 2026.  Figure~\ref{fig:qwen35-2b-warmup-ablation} pairs each score with its measured warmup time.  The score rises from 29.9 to 35.0 as warmup increases from 1.2 to 23.1 minutes.  Gains diminish at the upper end: moving from $\tau=0.20$ to $0.40$ adds only 0.2 percentage points while requiring 12.4 additional minutes.  The $\tau=0.00$ setting still incurs 1.2 minutes of warmup, rather than removing warmup entirely.

\begin{figure}[htbp]
\centering
\includegraphics[width=0.92\linewidth]{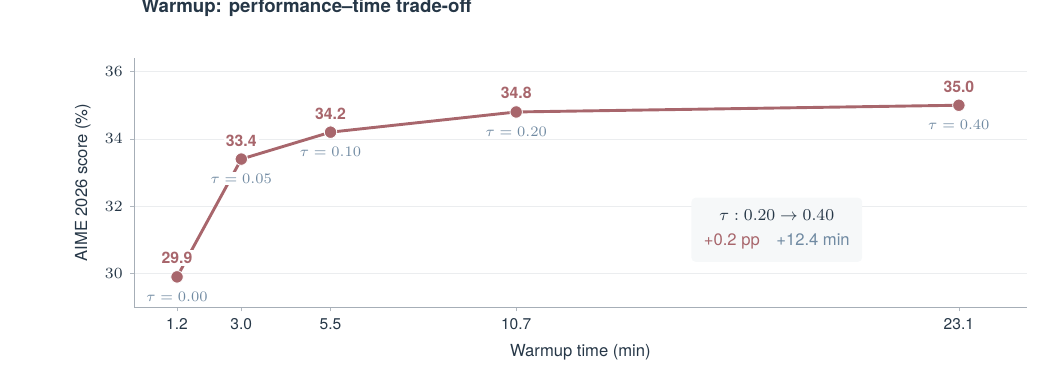}
\caption{Qwen3.5-2B \tfour{} warmup-parameter ablation on AIME 2026.  Scores are plotted against measured warmup time, with point labels identifying the score and $\tau$.}
\label{fig:qwen35-2b-warmup-ablation}
\end{figure}

\clearpage
\subsection{Earlier Cross-Task Comparison and Provenance}
\label{app:preliminary-provenance}

The available summary compares Base, Quiet-STaR (w/ GRPO), and \tfour{} on Qwen3-1.7B-Base \citep{yang2025qwen3}.  All trained variants use the same Quiet-STaR hidden-thought architecture.  The group-relative variant uses GRPO \citep{zelikman2024quietstar,shao2024deepseekmath}.  The mathematics group comprises MATH500, drawn from MATH \citep{hendrycks2021math}, GSM8K \citep{cobbe2021verifiers}, AMC23, and Minerva \citep{lewkowycz2022solving}.  The other tasks are MMLU \citep{hendrycks2021mmlu}, MMLU-Pro \citep{wang2024mmlupro}, and GPQA \citep{rein2024gpqa}.  This suite differs from the primary target suite in Section~\ref{sec:experimental-configuration}.  Its Quiet-STaR (w/ GRPO) baseline also uses group size 8 for both the benchmark summary and the training-time comparison.

Table~\ref{tab:preliminary-benchmarks} preserves all ten benchmark entries in the supplied summary.  The scores are numerical inputs to the original plotting script, rather than newly computed evaluations.  The Qwen3-1.7B initialization and shared architecture are author-confirmed, while the accompanying machine-readable materials do not record the training corpus, inference settings, training seeds, or exact \tfour{} implementation.  Consequently, these results are not attributed to either Qwen3.5 initialization or to the single-rollout conditional-moment protocol analyzed in this paper.

\begin{table}[htbp]
\centering
\small
\renewcommand{\arraystretch}{1.12}
\caption{Complete preliminary benchmark summary (\%).  Bold indicates the largest reported value in each row.  The three additional rows report pass@1.}
\label{tab:preliminary-benchmarks}
\begin{tabular*}{0.92\linewidth}{@{\extracolsep{\fill}}lrr>{\columncolor{tableaccent}}r@{}}
\toprule
\rowcolor{tableaccent}\textbf{Benchmark} & \textbf{Base} & \textbf{\shortstack{Quiet-STaR\\(w/ GRPO)}} & \textbf{\tfour{}} \\
\midrule
MATH500 & 48.17 & 57.96 & \textbf{58.10} \\
GSM8K & 54.47 & 72.42 & \textbf{75.00} \\
AMC23 & 26.19 & 30.76 & \textbf{31.70} \\
Minerva & 15.50 & 19.40 & \textbf{21.00} \\
\addlinespace[3pt]
MMLU & 50.61 & 42.40 & \textbf{55.80} \\
MMLU pass@1 & 45.10 & 40.49 & \textbf{51.90} \\
MMLU-Pro & 28.60 & 28.30 & \textbf{34.90} \\
MMLU-Pro pass@1 & 24.30 & 25.10 & \textbf{31.10} \\
GPQA & 25.60 & 26.80 & \textbf{28.70} \\
GPQA pass@1 & \textbf{27.00} & 25.20 & 26.80 \\
\bottomrule
\end{tabular*}
\end{table}

Figure~\ref{fig:benchmark-comparison} shows higher \tfour{} scores than Quiet-STaR (w/ GRPO) on all seven displayed tasks.  The difference is small on MATH500 (58.10 versus 57.96), but larger on GSM8K (75.00 versus 72.42) and MMLU (55.80 versus 42.40).  Relative to Base, Quiet-STaR (w/ GRPO) improves the four mathematics scores while its MMLU score decreases by 8.21 points, from 50.61 to 42.40.  \tfour{} exceeds Base on the seven displayed tasks.  The complete summary also contains an exception: on GPQA pass@1, \tfour{} scores 26.80 against Base's 27.00.  Thus the observed improvement is not uniform across all reported evaluation entries.  The Base scores are not uniformly depressed: on the seven main tasks they differ by at most 0.53 points from the Qwen3-1.7B-Base values reported with RLP \citep{hatamizadeh2025rlp}.  The conspicuous value is instead the Quiet-STaR (w/ GRPO) MMLU regression, which we treat as a run-specific observation rather than general evidence.

The three additional rows report pass@1.  The supplied summary does not identify their exact aggregation protocol.  The benchmark panel shows the seven main entries, while the table retains the three pass@1 values, including the unfavorable GPQA comparison with Base.  These point estimates motivate a controlled comparison but do not isolate predictive gating or complementary weighting.

Figure~\ref{fig:training-time} plots the supplied Quiet-STaR (w/ PPO), Quiet-STaR (w/ GRPO), and \tfour{} performance trajectories against training time.  Their common time-zero point denotes the shared Qwen3-1.7B-Base checkpoint.  The remaining points belong to the corresponding method.  The compact panel omits numeric ticks and intermediate markers, while the source JSON retains every coordinate.  Timing includes critic warmup, rollout generation, reward scoring, critic fitting, weight and auxiliary-head learning, validation, and rejected windows under matched hardware and evaluation settings.

\clearpage
\subsection{Recorded Per-Step Training Time}
\label{app:qwen35-step-time}

Figure~\ref{fig:qwen35-step-time} compares \tfour{} with the critic-free RLP baseline, which uses GRPO-style updates with group size 8 at both model scales.  It shows every supplied \texttt{step\_time\_seconds} value as a faint line, with 15-step rolling medians overlaid.  The two 2B RLP exports cover steps 1--200 and 201--500, respectively.  We join them at their original step indices and use all 500 records for the reported RLP mean.  Table~\ref{tab:qwen35-mean-step-time} uses the final continuous 2B \tfour{} segment (steps 301--600) for its mean, excluding the earlier startup and interrupted segments.  For \tfour{}, the figure retains all three column pairs in timestamp-defined order.  Because the second segment restarts its local step count, later segments are offset to form a monotone plotted step axis.  Dashed vertical lines mark the two joins.  The logarithmic time axis retains the initial high-cost segment without flattening the remaining measurements.

\begin{figure}[htbp]
\centering
\includegraphics[width=0.94\linewidth]{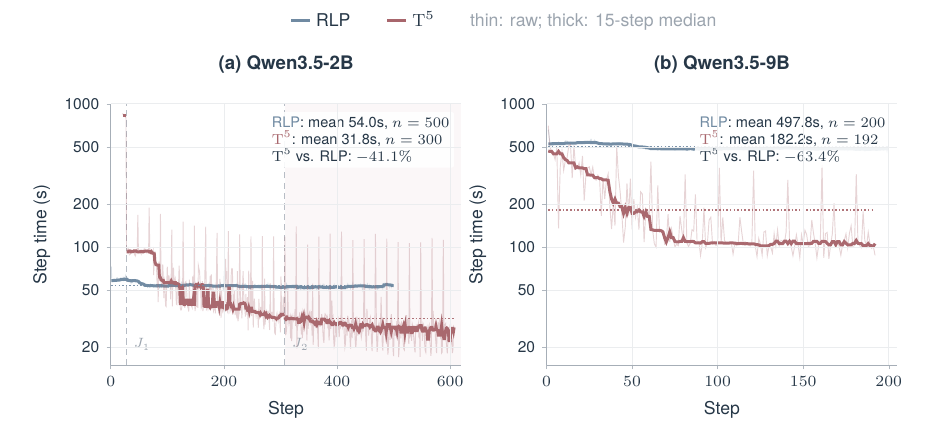}
\caption{Recorded per-step training times for RLP and \tfour{} on Qwen3.5-2B and Qwen3.5-9B.  Faint lines show raw timers, solid lines show 15-step rolling medians, and dotted lines mark the reported means.  The shaded 2B region is the final continuous \tfour{} segment used for its mean; $J_1$ and $J_2$ mark segment joins.  The logarithmic vertical axis reports seconds.}
\label{fig:qwen35-step-time}
\end{figure}

\clearpage
\subsection{Performance and Efficiency Across Model Scales}
\label{app:qwen35-scale-results}

\paragraph{Within-scale comparison.}
Figure~\ref{fig:qwen35-9b-benchmark-comparison} compares Base, NTP, and \tfour{} on Qwen3.5-9B using the seven-task metric convention of Table~\ref{tab:qwen35-2b-main-results}.  \tfour{} exceeds both baselines on every task.  Its unweighted mean is 92.71, compared with 86.86 for Base and 87.00 for NTP, giving gains of 5.86 and 5.71 percentage points, respectively.  The largest margin over NTP is on AIME 2026 (93 versus 76), while MATH-500 improves from an already high 96 to 97.  These results show that the method remains useful when applied to a stronger initialization.

\begin{figure}[htbp]
\centering
\includegraphics[width=\linewidth]{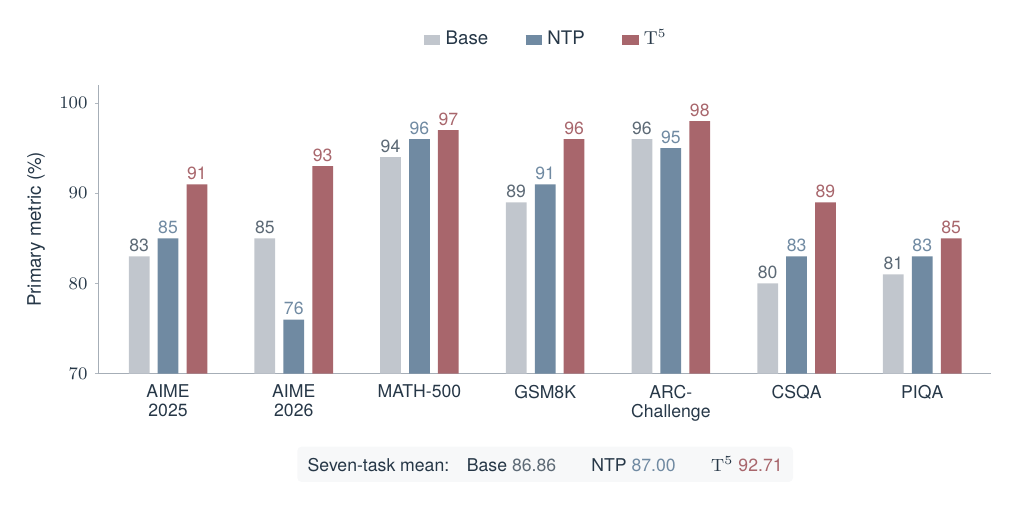}
\caption{Qwen3.5-9B comparison of Base, NTP, and \tfour{}.  Labels preserve the supplied scores in percent.  The four mathematics tasks use pass@1, and the three multiple-choice tasks use accuracy.  CSQA abbreviates CommonsenseQA.  The vertical axis starts at 70, and the footer reports unweighted means over all seven tasks.}
\label{fig:qwen35-9b-benchmark-comparison}
\end{figure}

\Needspace{34\baselineskip}
\paragraph{Cross-scale performance.}
Figure~\ref{fig:qwen35-2b-9b-common-methods} aligns the two model scales using only the three shared methods.  All three improve on every task at 9B.  The \tfour{} mean rises from 69.56 to 92.71, a 23.15-point increase.  Its margin over NTP remains positive at both scales, although it narrows from 13.26 to 5.71 points as the baseline scores rise.  The comparison therefore separates the benefit of a larger base model from the additional benefit of \tfour{} within each scale.

\begin{figure}[htbp]
\centering
\includegraphics[width=\linewidth]{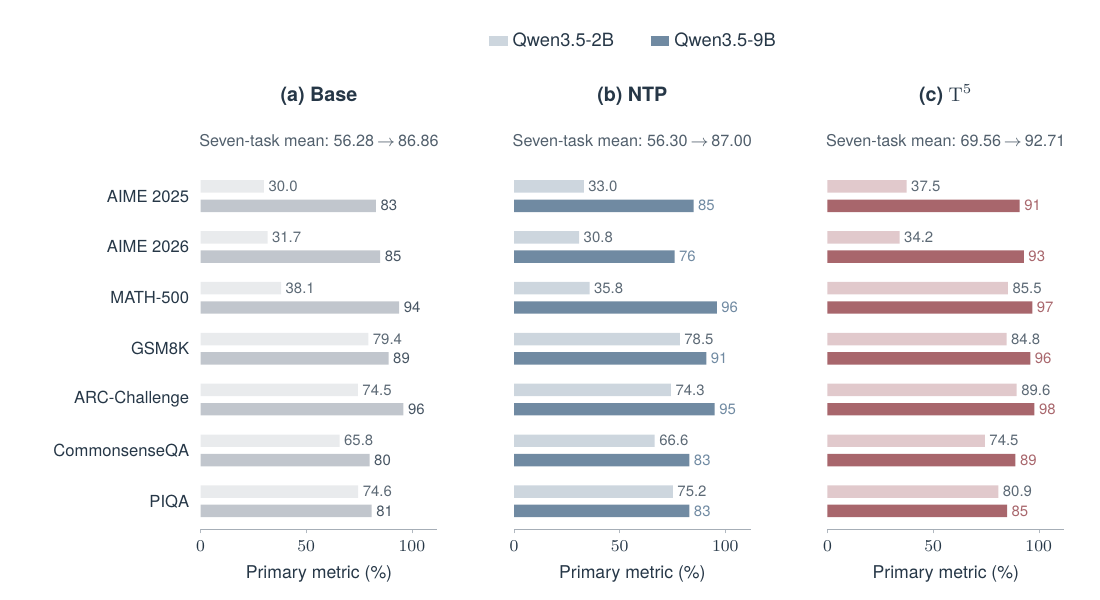}
\caption{Common-method comparison between Qwen3.5-2B and Qwen3.5-9B.  Each panel pairs a method's scores on the same seven tasks.  Light bars show 2B and darker bars show 9B, with identical score axes across panels.  The 2B bars use unrounded source values, with labels rounded to one decimal.  Headings report the seven-task mean at each scale.}
\label{fig:qwen35-2b-9b-common-methods}
\end{figure}

\paragraph{Relation to training efficiency.}
These results complement the timing comparison for the third research question.  In Table~\ref{tab:qwen35-mean-step-time}, the recorded mean step time grows from 31.79 to 182.18 seconds for \tfour{}, compared with 54.01 to 497.83 seconds for RLP.  Thus, the relative step-time reduction increases from 41.1\% at 2B to 63.4\% at 9B while \tfour{} retains gains over Base and NTP at the larger scale.  Together, the measurements describe performance and per-step cost as model size grows, rather than time to a common target score.

\clearpage
\subsection{Historical Naive-PPO Diagnostic}
\label{app:historical-naive-ppo}

Figure~\ref{fig:learning-dynamics} combines thought and mixed gains from an earlier 60-step naive PPO run, without \tfour{} or the new $R^2$ gate.  It evaluates 256 examples each from CommonsenseQA \citep{talmor2019commonsenseqa} and GSM8K \citep{cobbe2021verifiers} every ten steps.  The actor is frozen through step 20.  The retained record does not specify the allocation of these warmup steps between value heads and backbones.  The evaluation seed is fixed within the run.  Thought gain is $\ell_0-\ell_{\mathrm{thought}}$, while mixed gain substitutes the learned prediction mixture's loss.  The no-thought prediction is therefore the zero-gain reference by construction.  In this sense, the figure is an ablation-style diagnostic of prediction mode at a fixed checkpoint, not a comparison between separately trained Base and RL models.

\begin{figure}[htbp]
\centering
\includegraphics[width=\linewidth]{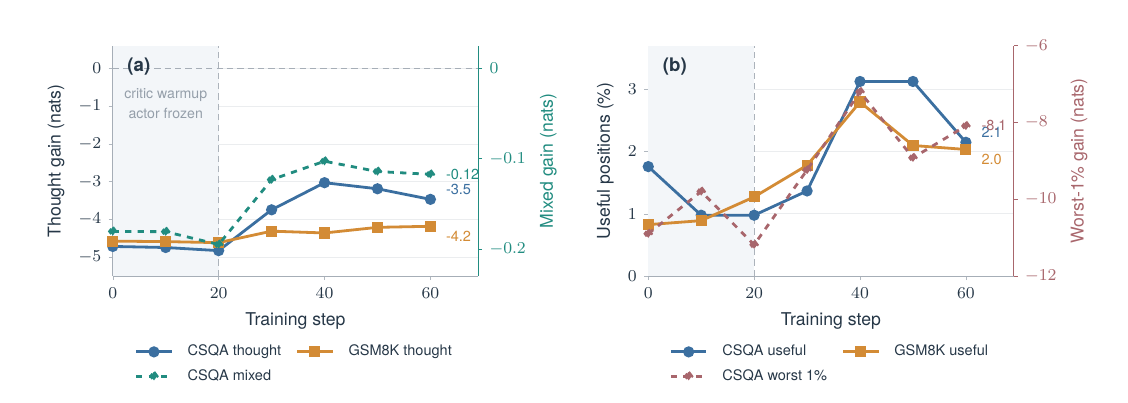}
\caption{Historical naive-PPO ablation-style diagnostic at seven checkpoints.  (a) Thought gain (left axis) and CSQA mixed gain (right axis).  (b) Useful-position fraction (left axis) and CSQA worst-1\% thought gain (right axis).  The no-thought prediction has zero gain by construction.  Shading marks the 20-step critic warmup with the actor frozen; endpoint labels are rounded values reconstructed from the original vector figure.}
\label{fig:historical-diagnostics-full}
\end{figure}

The redraw reconstructs the seven coordinates of each curve from the original vector PDF and its axis ticks.  These are figure-derived values, not recovered raw run logs.  Raw thought insertion remains below the zero reference on average, whereas prediction mixing substantially narrows the deficit and becomes slightly positive at the final checkpoint.  This historical contrast motivates qualification and mixing but does not evaluate \tfour{}.  In the current experiments, \tfour{} exceeds Base on all seven reported tasks at both 2B and 9B.  Extra stored decimal places preserve geometry and do not imply measurement precision.  No multi-seed statistics or causal attribution to critic staleness are inferred from this diagnostic.  Source assets, extraction metadata, and the plotting scripts are retained with the manuscript materials.

\newpage
\section{Why Predictive Accuracy Does Not Remove Drift}
\label{app:drift-foundations}
Predictive qualification and conditional calibration address different parts of the update.  We first separate policy reuse from value error, relate critic accuracy to advantage perturbations, and then show why accurate global prediction can coexist with harmful local drift.  This motivates the conditional-moment objective proved in Appendix~\ref{app:conditional-calibration}.

\subsection{What changes when old rollouts train a new policy?}
\label{app:policy-mismatch}
This subsection derives the distinction used in Section~\ref{sec:policy-mismatch}: policy lag changes the trajectories relevant to the objective, probability mismatch changes the correctness of recorded ratios, and value error changes the advantage estimate.  The derivations concern a finite undiscounted horizon, a fixed distribution of text items, fixed transitions, and a fixed scorer and reward normalizer.  Variable-length thoughts may be padded after termination with zero rewards and a deterministic dummy action.  All policy gradients below assume differentiable parameter-independent support and sufficient integrability to exchange derivatives and expectations.  The time index is included in the state.

For completeness, the return and ideal policy-gradient definitions used in the main text are
\begin{equation}
G_t=\sum_{\ell=0}^{T-t}\gamma^\ell r_{t+\ell}.
\label{eq:critic-return}
\end{equation}
For the undiscounted case $\gamma=1$, let $\tau$ denote a complete thought trajectory, including its text item, and write $u_t=\nabla_\theta\log\pi_\theta(a_t\mid s_t)$.  With scoring and transitions fixed, the expected-return objective and its policy gradient are
\begin{equation}
J(\theta)=\E_{\tau\sim\pi_\theta}[G_1],\qquad
\nabla_\theta J(\theta)=\E_{\tau\sim\pi_\theta}\!\left[\sum_t G_t u_t\right].
\label{eq:true-policy-gradient}
\end{equation}
\subsubsection{A fixed reward, but a changing distribution of thoughts}
Let $\tau$ contain the text item and the complete sampled thought.  Its probability factors into the text distribution, transitions, and action probabilities.  Only the last depends on $\theta$.  Write $G_1=\sum_{k=1}^T r_k$, $J(\theta)=\E_{\pi_\theta}G_1$, and $u_t=\nabla_\theta\log\pi_\theta(a_t\mid s_t)$.
\begin{lemma}[Policy gradient from trajectory likelihood]
\label{lem:trajectory-policy-gradient}
Under the fixed-reward conditions above,
\begin{equation}
\nabla_\theta J=\E_{\pi_\theta}\!\left[G_1\sum_tu_t\right]
=\E_{\pi_\theta}\!\left[\sum_tG_tu_t\right].
\label{eq:trajectory-gradient-proof}
\end{equation}
\end{lemma}
\begin{proof}
Write $p_\theta(\tau)$ for the trajectory probability and differentiate the expectation:
\begin{align}
\nabla_\theta\sum_\tau p_\theta(\tau)G_1(\tau)
&=\sum_\tau p_\theta(\tau)G_1(\tau)\nabla_\theta\log p_\theta(\tau),\\
\nabla_\theta\log p_\theta(\tau)&=\sum_t\nabla_\theta\log\pi_\theta(a_t\mid s_t).
\end{align}
This proves the first expression.  For the second, separate rewards before action $t$ from rewards at and after it.  Conditional on the history before $a_t$, earlier rewards are fixed, while
\begin{equation}
\E_{a_t\sim\pi_\theta}[u_t\mid s_t]
=\sum_a\nabla_\theta\pi_\theta(a\mid s_t)
=\nabla_\theta1=0.
\end{equation}
Earlier rewards therefore contribute zero.  The term for action $t$ retains only $G_t=\sum_{k=t}^T r_k$.
\end{proof}
This is why it is more precise to say that a return sample was generated under a policy than to label the score of a fixed trajectory as a different reward for each policy.  The same thought receives the same score within the window.  What changes is its likelihood, and the expected quality of continuations from each partial thought.

Now let $q$ be the actual behavior law, fixed while differentiating the actor.  Naively replaying old trajectories gives
\begin{equation}
g_{\mathrm{naive}}(\theta)=\E_{\tau\sim q}\!\left[\sum_tG_tu_t\right],
\label{eq:naive-replay-gradient}
\end{equation}
which differs from Eq.~\eqref{eq:trajectory-gradient-proof} through the trajectory distribution.  If $q$ covers the actor's trajectory support, the full likelihood ratio is
\begin{equation}
W_\theta(\tau)=\prod_{t=1}^T\frac{\pi_\theta(a_t\mid s_t)}{q(a_t\mid s_t)}.
\label{eq:full-trajectory-ratio}
\end{equation}
The initial text distribution and transitions cancel.  Provided the weighted expectation is integrable, multiplying each trajectory by this ratio restores the exact gradient:
\begin{equation}
\E_q\!\left[W_\theta(\tau)\sum_tG_tu_t\right]
=\sum_\tau p_q(\tau)\frac{p_\theta(\tau)}{p_q(\tau)}\sum_tG_tu_t
=\nabla_\theta J(\theta).
\label{eq:trajectory-importance-gradient}
\end{equation}
Long products can concentrate weight on a few trajectories even when individual token ratios are moderate.  This variance cost motivates local surrogates and bounded reuse, rather than assuming that token-level reweighting fully corrects a sequence distribution.

\subsubsection{Why an old-policy advantage is intentional in PPO}
For this comparison assume $q=\pi_{\mathrm{old}}$ exactly, and let $d_t^\pi$ be the distribution of prefixes before token $t$ under policy $\pi$.  Define $Q_t^\pi(s,a)=\E_\pi[G_t\mid s_t=s,a_t=a]$, $V_t^\pi(s)=\E_\pi[G_t\mid s_t=s]$, and $A_t^\pi=Q_t^\pi-V_t^\pi$, retaining the fixed-text conditioning convention.  The ideal $L_{\mathrm{old}}$ in Eq.~\eqref{eq:old-policy-surrogate} uses exact, unnormalized $A_t^{\mathrm{old}}$ and every action in each old trajectory.
PPO deliberately makes a local approximation: before clipping and normalization, its ideal old-policy surrogate is
\begin{equation}
L_{\mathrm{old}}(\theta)=\E_{\tau\sim\pi_{\mathrm{old}}}\!\left[\sum_t\rho_t A_t^{\mathrm{old}}\right],
\qquad \left.\nabla_\theta L_{\mathrm{old}}(\theta)\right|_{\theta=\theta_{\mathrm{old}}}
=\nabla_\theta J(\theta_{\mathrm{old}}).
\label{eq:old-policy-surrogate}
\end{equation}
\begin{proposition}[The old-policy surrogate matches the gradient at its origin]
\label{prop:surrogate-tangency}
For the fixed-reward, full-trajectory setting above,
\begin{equation}
\left.\nabla_\theta L_{\mathrm{old}}(\theta)\right|_{\theta=\theta_{\mathrm{old}}}
=\nabla_\theta J(\theta_{\mathrm{old}}).
\end{equation}
\end{proposition}
\begin{proof}
At the old parameter, $\rho_t=1$ and $\nabla_\theta\rho_t=u_t$.  The state value contributes zero by the conditional score identity, hence
\begin{equation}
\left.\nabla_\theta L_{\mathrm{old}}\right|_{\theta_{\mathrm{old}}}
=\E_{\mathrm{old}}\sum_tA_t^{\mathrm{old}}u_t
=\E_{\mathrm{old}}\sum_tQ_t^{\mathrm{old}}u_t.
\end{equation}
The factor $u_t$ is fixed after conditioning on $(s_t,a_t)$, so averaging $G_tu_t$ over the remaining continuation gives $Q_t^{\mathrm{old}}u_t$.  Reversing this conditional expectation and applying Lemma~\ref{lem:trajectory-policy-gradient} proves the result.
\end{proof}
Away from this starting point, exact current-action reweighting leaves the following two expressions:
\begin{align}
\nabla_\theta L_{\mathrm{old}}(\theta)
&=\sum_t\E_{s\sim d_t^{\mathrm{old}},\,a\sim\pi_\theta}[A_t^{\mathrm{old}}(s,a)u_t],\\
\nabla_\theta J(\theta)
&=\sum_t\E_{s\sim d_t^{\pi_\theta},\,a\sim\pi_\theta}[A_t^{\pi_\theta}(s,a)u_t].
\label{eq:surrogate-distribution-gap}
\end{align}
The first expression keeps old prefixes and old continuation values.  The second uses the prefixes and continuations induced by the current actor.  A token ratio changes neither of those other factors.  Thus old-policy advantages are the reference for a local policy improvement, not evidence that the critic must already predict the updated policy perfectly.  This is the standard surrogate distinction underlying TRPO and PPO \citep{schulman2015trpo,schulman2017ppo}.

The identity at the origin concerns exact old-policy advantages (or unbiased Monte Carlo estimates of them) and full-trajectory summation.  Normalized GAE, action-dependent mixture weights, and clipping away from the origin change that estimator.  Section~\ref{sec:conditional-calibration} controls a component of the actual mixed PPO gradient, not an exact $\nabla J$.  Relating them requires matching the occupation measure and time weights.  Appendix~\ref{app:restart-contract} fixes the aggregation for the auxiliary objective.

\subsubsection{Value-learning error and a moving value target}
Fix a prefix and scoring item.  Adding and subtracting $V^q$ gives the decomposition in Section~\ref{sec:critic-update-mismatch}.  The first term compares the fitted critic with the current rollout return target.  Its fitting data may come from an earlier distribution.  The second compares two different continuation policies, even if the first prediction were exact.  For example, suppose a remaining binary decision returns one on success and zero otherwise.  If the old policy succeeds with probability $0.2$ and the new policy with probability $0.8$, a perfect old critic predicts $0.2$.  Its difference from the new value is $-0.6$, entirely due to policy change, not failed old-policy fitting.  More regression on old returns keeps the optimum at $0.2$.

To expose historical target lag, let $q_\beta$ denote one earlier continuation law under the same fixed reward.  Then
\begin{equation}
V_\psi-V^q=(V_\psi-V^{q_\beta})+(V^{q_\beta}-V^q).
\label{eq:historical-critic-target}
\end{equation}
The first term is fitting error relative to that earlier target, and the second is target lag.  A critic trained on mixed history need not equal the value of any single policy.  $q_\beta$ illustrates one source rather than identifying an exact critic age.  Scorer or reward-scale changes introduce further target changes and are excluded from this fixed-reward comparison.

There is an additional information distinction in this paper.  Let $W$ be the observed future text used for scoring, and let $V^q(s,W)=\E_q[G_t\mid s,W]$.  Under ordinary squared regression, an unrestricted predictor that sees only $s$ is minimized at
\begin{equation}
\bar V^q(s)=\E_{W\mid s}[V^q(s,W)].
\label{eq:visible-regression-target}
\end{equation}
Indeed, expanding around $\bar V^q(s)$ makes the conditional regression risk a variance term plus $(V_i(s)-\bar V^q(s))^2$.  The item-level error further separates as
\begin{equation}
V_i(s)-V^q(s,W)
=[V_i(s)-\bar V^q(s)]+[\bar V^q(s)-V^q(s,W)].
\end{equation}
Only the first part is visible-state fitting error.  The second reflects information withheld from the critic.  Appendix~\ref{app:information-limit} quantifies its variance floor.  Predictive qualification consequently checks useful out-of-sample return structure, not exact reconstruction of every item's oracle value.

\subsubsection{Baseline cancellation under exact action ratios}
Freeze a state-only baseline $b(s)$ before drawing the evaluated action and continuation.  At any fixed prefix, assume exact action ratios $\rho=\pi_\theta/q$ and full actor-support coverage.  Then, regardless of how different the actor is from $q$,
\begin{equation}
\E_q[\rho b(s)u_t\mid s]
=b(s)\sum_a\pi_\theta(a\mid s)\nabla_\theta\log\pi_\theta(a\mid s)=0.
\label{eq:offpolicy-baseline-cancellation}
\end{equation}
The baseline need not equal either $V^q$ or $V^{\pi_\theta}$.  This cancellation concerns the baseline contribution only, not the old-state and old-continuation errors in Eq.~\eqref{eq:surrogate-distribution-gap}.  Fitting a baseline on the evaluated actions and then detaching it is not equivalent to fixing it independently beforehand.

Two operations can break the cancellation.  With an inexact recorded denominator, the remaining action weight is $q/\pi_{\mathrm{old}}$.  With clipping, a sample-dependent coefficient $I$ selects which directions remain active.  Even without clipping the first effect gives
\begin{equation}
\E_q[\rho u_t\mid s]
=\sum_a\pi_\theta(a\mid s)\frac{q(a\mid s)}{\pi_{\mathrm{old}}(a\mid s)}u_t,
\label{eq:denominator-mismatch-score}
\end{equation}
which need not vanish.  GAE also introduces future value errors rather than merely subtracting the current state baseline, as derived in Proposition~\ref{prop:gae-propagation}.  Section~\ref{sec:twin-calibration} addresses the resulting conditional-mean component with the actual mixed PPO mask, not every source of surrogate error.

These distinctions apply to both small and large RL problems.  Long language-model continuations, sparse terminal signals, expensive rollout refresh, and seldom-repeated exact prefixes make policy lag and value generalization practically relevant.  They do not imply that PPO is necessarily less reliable than GRPO.  GRPO removes learned-value regression but still uses sampled advantages and must handle rollout reuse and recorded probabilities \citep{shao2024deepseekmath}.  Finally, our scorer is frozen only within a window.  If one instead differentiated a changing reward $G_1(\theta)$ as part of the objective, the exact derivative would include $\E_{\pi_\theta}[\nabla_\theta G_1(\theta)]$ in addition to the score-function term.  Our actor update treats scores as fixed.  Updating the scorer starts a new objective window rather than differentiating through that reward term.

\subsection{Value error and predictive qualification}
\label{app:value-gradient-proof}
\label{app:predictive-qualification}
\label{sec:predictive-analysis}
This subsection explains why both critics are qualified separately.  All expectations here use a fixed probability distribution of text items and valid-action prefixes, followed by continuations from $q$.  Conditioning on $s$ also fixes the scoring text, as in Section~\ref{sec:problem-setup}.  The oracle $V^q$ is the corresponding conditional return, whether or not a visible-prefix critic can represent it.  Fix $\theta$, the recorded denominator, scorer, critics, and normalizers.  Assume differentiability, integrability of the gradients below, and finite $C_\rho^2=\E_q[\rho^2\|\nabla_\theta\log\pi_\theta(a\mid s)\|_2^2]$.  Write $e_i=V_i-V^q$ and $u=\nabla_\theta\log\pi_\theta(a\mid s)$.  Population predictive scores use this probability measure.  Moving to an unnormalized occupation measure requires rescaling the gradients, risks, and $C_\rho$ together.

\subsubsection{Clipping does not amplify advantage perturbations}
First consider shared-return Monte Carlo advantages $A_i=G-V_i(s)$.  Compare a detached convex mixture, allowing $w=w(s,a)\in[0,1]$, with the oracle advantage $G-V^q(s)$ under the same normalizer.
\begin{proposition}[Value error controls update error]
\label{prop:value-gradient-bound}
\begin{equation}
\|g_{\mathrm{twin}}-g_{\mathrm{oracle}}\|_2
\leq\frac{C_\rho}{\sigma_{\mathrm{pilot}}}\sqrt{\E_s[e_1(s)^2+e_2(s)^2]}.
\label{eq:value-gradient-bound}
\end{equation}
\end{proposition}

For a scalar normalized advantage $x$, the per-sample clipped-PPO gradient is $\rho u H_\rho(x)$, where, away from clipping boundaries,
\begin{equation}
H_\rho(x)=
\begin{cases}
\max(x,0),&\rho<\rho_-,\\
x,&\rho_-<\rho<\rho_+,\\
\min(x,0),&\rho>\rho_+.
\end{cases}
\qquad |H_\rho(x)-H_\rho(y)|\leq|x-y|.
\label{eq:ppo-advantage-lipschitz}
\end{equation}
At a clipping boundary use one fixed convex combination of the two adjacent branch maps.  The same convention is used for the two advantages being compared.
\begin{lemma}[The active advantage is nonexpansive]
\label{lem:ppo-nonexpansive}
For any fixed ratio $\rho$ and real numbers $x,y$,
\begin{equation}
|H_\rho(x)-H_\rho(y)|\leq|x-y|.
\end{equation}
\end{lemma}
\begin{proof}
For $H(x)=\max(x,0)$, two positive inputs retain their original distance, two negative inputs both become zero, and inputs on opposite sides of zero become closer because the negative part is removed.  Thus their distance cannot increase.  The proof for $\min(x,0)$ is identical with signs reversed.  The identity map preserves distance.  A convex combination of such maps also obeys the bound by the triangle inequality.  This covers the boundary convention.
\end{proof}
The lemma controls a switch of clipping branch without pretending that the two masks stay equal.

\begin{proof}[Proof of Proposition~\ref{prop:value-gradient-bound}]
For raw advantages $A,A_o$ with the same normalizer, subtract their gradient contributions and move the norm inside the expectation:
\begin{align}
\|g(A)-g(A_o)\|_2
&\leq\E_q[\rho\|u\|_2|H_\rho(\widetilde A)-H_\rho(\widetilde A_o)|]\\
&\leq\frac1{\sigma_{\mathrm{pilot}}}\E_q[\rho\|u\|_2|A-A_o|]\\
&\leq\frac{C_\rho}{\sigma_{\mathrm{pilot}}}\sqrt{\E_q[(A-A_o)^2]}.
\label{eq:advantage-gradient-transfer}
\end{align}
The second line uses Lemma~\ref{lem:ppo-nonexpansive}.  The last uses Cauchy--Schwarz on the two nonnegative factors $\rho\|u\|_2$ and $|A-A_o|$.

For Monte Carlo advantages, subtracting the two baselines cancels the shared return: $A-A_o=-we_1-(1-w)e_2$.  The square of this mixture satisfies
\begin{equation}
w e_1^2+(1-w)e_2^2-[we_1+(1-w)e_2]^2
=w(1-w)(e_1-e_2)^2\geq0.
\end{equation}
Therefore $(A-A_o)^2\leq we_1^2+(1-w)e_2^2\leq e_1^2+e_2^2$.  Insert this into Eq.~\eqref{eq:advantage-gradient-transfer} to obtain the proposition.
\end{proof}
The rollout distribution, clipping rule, and normalizer stay the same in this comparison.  The pointwise convex-error inequality remains valid for a detached action-dependent weight.  It does not invoke state-baseline cancellation.  Expectations use the declared rollout occupation measure.  Prediction risk under another distribution bounds actor-distribution error only with coverage control.

\subsubsection{Predictive scores and independent validation}
A collapsed predictor illustrates why the admission rule in Section~\ref{sec:prediction-gated-training} tests both critics separately.  Consider equally represented held-out states $s_\pm$ with $G(s_\pm)=\pm r_\star$, $r_\star>0$.  Let $V_1(s_\pm)=\pm r_\star$ and $V_2(s_\pm)=0$.  Then
\begin{equation}
\operatorname{MSE}(V_1)=0,\quad \operatorname{MSE}(V_2)=r_\star^2,
\quad (R_1^2,R_2^2)=(1,0).
\label{eq:critic-collapse-example}
\end{equation}
As $r_\star\to0$, the second critic's loss and twin disagreement vanish, although it has learned no return variation.  Even the average predictor attains $R^2=3/4$.  Testing the critics separately distinguishes this case from two qualified predictors, whereas a fixed warmup duration or small fitting loss does not.  Let $\mathcal E_i=\E_q[(G-V_i(s))^2]$, $S_G^2=\Var_q(G)>0$, and $R_{i,\mathrm{pop}}^2=1-\mathcal E_i/S_G^2$.  All three quantities use the same distribution, with fixed critics and finite return moments.
\begin{lemma}[Return error includes value error]
\label{lem:return-risk}
For each critic,
\begin{equation}
\mathcal E_i=\E_s[e_i(s)^2]+\E_s\Var_q(G\mid s).
\end{equation}
\end{lemma}
\begin{proof}
Write $G-V_i=(G-V^q)-e_i$ and expand the square.  Conditional on $s$, the cross term is $-2e_i\E_q[G-V^q\mid s]=0$, because $V^q=\E_q[G\mid s]$.  The remaining two terms are the conditional variance of $G$ and $e_i^2$.  Average them over states.
\end{proof}
Summing the identity for the two critics and substituting the definition of their population scores gives
\begin{align}
\E_s[e_1^2+e_2^2]
&=S_G^2(2-R_{1,\mathrm{pop}}^2-R_{2,\mathrm{pop}}^2)-2\E_s\Var_q(G\mid s)\\
&\leq S_G^2(2-R_{1,\mathrm{pop}}^2-R_{2,\mathrm{pop}}^2).
\label{eq:population-prediction-transfer}
\end{align}
Return-prediction risk measures value error plus random variation among continuations.  It is therefore conservative for Proposition~\ref{prop:value-gradient-bound}.  A global score controls an average, leaving room for the rare-state example in Appendix~\ref{app:twin-disagreement}.

The empirical $R_i^2$ in Eq.~\eqref{eq:holdout-r2} is a diagnostic unless accompanied by generalization control.  One sufficient construction works directly with squared residuals.  For critics fixed before $N_h$ independent holdout draws, suppose $|G-V_i(s)|\leq K_V$ on the evaluated distribution.  With probability at least $1-\delta$, simultaneously for both critics,
\begin{equation}
\mathcal E_i\leq\widehat{\mathcal E}_i+
K_V^2\sqrt{\frac{\log(2/\delta)}{2N_h}},
\qquad
\widehat{\mathcal E}_i=\frac1{N_h}\sum_j(G_j-V_i(s_j))^2.
\label{eq:value-risk-confidence}
\end{equation}
On this event, $\E_s(e_1^2+e_2^2)$ is bounded by the sum of the two upper estimates.  A gradient certificate additionally requires a bound on $C_\rho$.  Correlated tokens from one trajectory are not independent holdout draws.  Clustered validation and repeated qualification require appropriate concentration bounds and a declared total failure probability.

\subsubsection{How value errors propagate through GAE}
Compare learned and oracle GAEs on the same trajectory of realized length $L$, with the same $\lambda=\lambda(L)$, $0\leq\gamma,\lambda\leq1$, and zero terminal bootstrap.  Let $\zeta_{i,t}=A_{i,t}-A_{o,t}$ and $w_k=\gamma(1-\lambda)(\gamma\lambda)^{k-1}$.
\begin{proposition}[GAE propagates value error along the suffix]
\label{prop:gae-propagation}
For each realized trajectory and pre-action position $t$,
\begin{align}
A_{i,t}-A_{o,t}&=-e_i(s_t)+\gamma(1-\lambda)\sum_{k=1}^{L-t}(\gamma\lambda)^{k-1}e_i(s_{t+k}),
\label{eq:gae-error-in-ppo}\\
\zeta_{i,t}&=-e_i(s_t)+\sum_{k=1}^{L-t}w_k e_i(s_{t+k}).
\label{eq:gae-value-error-propagation}
\end{align}
\end{proposition}
\begin{proof}
Subtract the oracle TD error from the learned TD error.  The reward cancels, leaving $\gamma e_i(s_{t+1})-e_i(s_t)$.  Summing these differences with the same GAE weights gives
\begin{equation}
\zeta_{i,t}=\sum_{k=0}^{L-t}(\gamma\lambda)^k
[\gamma e_i(s_{t+k+1})-e_i(s_{t+k})].
\end{equation}
The coefficient of the current error is $-1$.  For an intermediate error $e_i(s_{t+k})$, $k\geq1$, its two occurrences have net coefficient $\gamma(\gamma\lambda)^{k-1}-(\gamma\lambda)^k=w_k$.  The final error is zero by terminal bootstrapping.  Collecting coefficients proves the identity.
\end{proof}
The current error enters directly.  Later errors enter with nonnegative, decaying weights.  Let $W_t=\sum_kw_k$.  For $\gamma\lambda<1$, its geometric sum is at most $\gamma(1-\lambda)/(1-\gamma\lambda)\leq1$.  If $\gamma\lambda=1$, every $w_k$ is zero.  Cauchy--Schwarz applied to vectors $(1,\sqrt{w_1},\ldots)$ and $(-e_i(s_t),\sqrt{w_1}e_i(s_{t+1}),\ldots)$ gives
\begin{equation}
\zeta_{i,t}^2\leq(1+W_t)\left[e_i(s_t)^2+\sum_{k=1}^{L-t}w_k e_i(s_{t+k})^2\right].
\label{eq:gae-error-moment}
\end{equation}
For the convex mixture, the error is $w\zeta_{1,t}+(1-w)\zeta_{2,t}$, with squared magnitude at most $\zeta_{1,t}^2+\zeta_{2,t}^2$.  Applying Eq.~\eqref{eq:advantage-gradient-transfer} therefore yields
\begin{equation}
\|g_{\mathrm{twin}}^{\mathrm{GAE}}-g_{\mathrm{oracle}}^{\mathrm{GAE}}\|_2
\leq\frac{C_\rho}{\sigma_{\mathrm{pilot}}}\sqrt{\E_q[\zeta_{1,t}^2+\zeta_{2,t}^2]}.
\label{eq:gae-gradient-bound}
\end{equation}
Uniform value error $|e_i|\leq\epsilon_i$ on the evaluated prefixes gives $|\zeta_{i,t}|\leq2\epsilon_i$.  At $\lambda=1$, only the current error remains.  These calculations explain why general GAE also needs accurate later-prefix values.  The oracle comparison uses matched GAE and does not assume that its conditional mean vanishes.

\subsection{Two distinct consistency conditions for cancellation}
\label{app:two-consistencies}
The factors in $\Delta_\mu(s)=\mu_{\widetilde A}(s)\,\mu_v(s)$ concern different objects.  \emph{Sampling--update consistency} asks whether action reweighting recovers the actor's zero-mean score.  \emph{Advantage--sampling consistency} asks whether the advantage is centered under the law supplying its actions and returns.  These are logically distinct conditions, not a claim of statistical independence between the factors.

Fix a prefix and its scoring item, and condition on all snapshots chosen independently of the evaluated continuation.  Write $u=\nabla_\theta\log\pi_\theta(a\mid s)$, $\rho=\pi_\theta/\pi_{\mathrm{old}}$, $\mu_{\widetilde A}=\E_q[\widetilde A\mid s]$, and $\mu_v=\E_q[\rho Iu\mid s]$.  Assume differentiable parameter-independent actor support covered by $q$, positive recorded probabilities, exchangeable differentiation and summation, and finite required moments.  Normalization statistics $\mu_{\mathrm{pilot}}$ and $\sigma_{\mathrm{pilot}}>0$ are fixed.

\begin{proposition}[Two routes to a zero mean-induced term]
\label{prop:two-cancellation-conditions}
Under these fixed-prefix conditions:
\begin{enumerate}
\item If $q=\pi_{\mathrm{old}}$ and $I=1$, then $\mu_v(s)=0$ for any actor parameter satisfying the assumptions, without requiring $\pi_\theta=\pi_{\mathrm{old}}$.  The same result holds for any covering sampler $q$ if the ratio is instead exactly $\pi_\theta/q$.
\item The exact advantage $A^q(s,a)=Q^q(s,a)-V^q(s)$ has zero conditional mean under $q$.  For a Monte Carlo estimate $A_i=G-V_i(s)$ with $V^q(s)=\E_q[G\mid s]$, its normalized conditional mean is
\begin{equation}
\mu_{\widetilde A_i}(s)=\frac{V^q(s)-V_i(s)-\mu_{\mathrm{pilot}}}{\sigma_{\mathrm{pilot}}}.
\label{eq:mc-conditional-offset}
\end{equation}
Thus $V_i=V^q$ and $\mu_{\mathrm{pilot}}=0$ give $\mu_{\widetilde A_i}(s)=0$, regardless of the actor ratio or clipping mask.
\end{enumerate}
Consequently, either $\mu_v(s)=0$ or $\mu_{\widetilde A}(s)=0$ suffices for $\Delta_\mu(s)=0$.  Both need not hold.
\end{proposition}
\begin{proof}
For the first part, the exact ratio changes the action measure, leaving
\begin{align}
\mu_v(s)&=\sum_a q(a\mid s)\frac{\pi_\theta(a\mid s)}{q(a\mid s)}
\nabla_\theta\log\pi_\theta(a\mid s)\\
&=\sum_a\pi_\theta(a\mid s)\nabla_\theta\log\pi_\theta(a\mid s)
=\nabla_\theta\sum_a\pi_\theta(a\mid s)=0.
\label{eq:exact-reweighting-zero-score}
\end{align}
This is the score-function identity, not an equality-of-policies assumption.  For the second part, define $Q^q(s,a)=\E_q[G\mid s,a]$.  Iterated expectation gives $V^q(s)=\sum_a q(a\mid s)Q^q(s,a)$, so
\begin{equation}
\E_{a\sim q}[A^q(s,a)\mid s]
=\sum_a q(a\mid s)Q^q(s,a)-V^q(s)=0.
\label{eq:own-policy-zero-advantage}
\end{equation}
The sampled return residual $G-V^q(s)$ also has zero conditional mean, although it need not equal $A^q(s,a)$ on each continuation.  Replacing $V^q$ by the frozen $V_i$ and applying the fixed affine normalizer yields Eq.~\eqref{eq:mc-conditional-offset}.  Finally, multiply the two factors in Eq.~\eqref{eq:ppo-drift-decomposition}.
\end{proof}

\paragraph{Why practical PPO can have $\mu_v(s)\ne0$.}
The active mask depends on the sampled advantage and hence may depend on the continuation, not just the action.  Let $\bar I(s,a)=\E_q[I\mid s,a]$.  Using the recorded ratio rather than assuming it is exact gives
\begin{equation}
\mu_v(s)=\sum_a\pi_\theta(a\mid s)
\frac{q(a\mid s)}{\pi_{\mathrm{old}}(a\mid s)}
\bar I(s,a)u(s,a).
\label{eq:effective-action-score}
\end{equation}
When $q=\pi_{\mathrm{old}}$ and $I=1$, all directions receive unit weight and cancel.  An incorrect denominator or selective clipping can make the effective weights nonuniform and prevent cancellation.  Neither necessarily makes $\mu_v$ nonzero: the weighted directions may still cancel.  Thus $\mu_v=0$ is not an if-and-only-if diagnostic of sampling--update consistency, and policy lag alone does not invalidate the first part of the proposition.  Appendix~\ref{app:fisher-drift} bounds the effect of uneven weights.  A zero population mean also does not make a finite replay's empirical score mean exactly zero.

\paragraph{What $\mu_{\widetilde A}(s)$ measures, and where the Monte Carlo interpretation ends.}
With $\mu_{\mathrm{pilot}}=0$, Eq.~\eqref{eq:mc-conditional-offset} identifies $\mu_{\widetilde A_i}$ as the scaled negative value error, $(V^q-V_i)/\sigma_{\mathrm{pilot}}$.  The value is conditional on the scoring item.  A critic seeing only the visible prefix may not represent it exactly (Appendix~\ref{app:policy-mismatch}).  For an exact advantage, iterated expectation also gives zero global population mean, so population normalization preserves conditional centering.  An independently estimated and frozen pilot mean need not be zero: even an exact Monte Carlo critic then leaves $\mu_{\widetilde A_i}=-\mu_{\mathrm{pilot}}/\sigma_{\mathrm{pilot}}$.  Recomputing statistics on the evaluated batch introduces dependence and does not justify treating them as fixed in this identity.

For GAE, value errors also enter through later prefixes, as Proposition~\ref{prop:gae-propagation} shows.  Exact values with a fixed trace parameter give zero-mean oracle GAE by iterated expectation of the zero-mean TD residuals.  A trace selected from the realized continuation length can correlate with those residuals, so exact values alone do not establish the same centering.  Our length-adaptive implementation therefore estimates the conditional mean of the actual GAE signal rather than assuming it equals the current state's value error.

An inaccurate state baseline is harmless to the ideal, correctly reweighted and unclipped score gradient because $\mu_v=0$, not because its advantage necessarily has $\mu_{\widetilde A}=0$.  Conversely, conditional centering eliminates the mean-induced term even if $\mu_v\ne0$.  \tfour{} targets the mixed GAE's conditional mean through the dual objective in Section~\ref{sec:twin-calibration}.  Action-dependent weights introduce the covariance in Eq.~\eqref{eq:mixture-conditional-mean}, so the mean is not a convex combination of two endpoint means.  Its population risk is controlled only to the extent justified by the auxiliary approximation, optimization, and statistical errors in Theorem~\ref{thm:complementary-drift}.  Mixing changes its own PPO mask and can alter the covariance term.  Eliminating $\Delta_\mu$ neither recovers the complete current-policy gradient nor establishes baseline invariance under action-dependent weighting or clipping.

\subsection{Why a fixed minimum does not calibrate advantages}
\label{app:twin-disagreement}
We analyze minimum aggregation as a comparator to the calibrated convex weights in Section~\ref{sec:twin-calibration}.  This calculation does not attribute that implementation to the historical naive-PPO diagnostic.  For two GAEs, let $\bar A=(A_1+A_2)/2$ and $\Delta=|A_1-A_2|/2$.  Then
\begin{equation}
A_{\min}=\bar A-\Delta,\qquad
\mu_{\widetilde A_{\min}}(s)=\frac{\E_q[\bar A\mid s]-\E_q[\Delta\mid s]-\mu_{\mathrm{pilot}}}{\sigma_{\mathrm{pilot}}}.
\label{eq:twin-drift-coefficient}
\end{equation}
The minimum moves the conditional mean downward rather than selecting its distance to zero.  For Monte Carlo advantages with zero terminal value,
\begin{equation}
A_{\min}=G-\max_iV_i(s),\qquad
\mu_{\widetilde A_{\min}}(s)=\frac{V^q(s)-\max_iV_i(s)-\mu_{\mathrm{pilot}}}{\sigma_{\mathrm{pilot}}}.
\label{eq:twin-mc-special-case}
\end{equation}
A constant offset disappears if normalization is recomputed from that shifted batch, but a state-dependent offset need not.  With a frozen normalizer, neither disappears automatically.  These identities explain the limitation of pessimistic aggregation.  The proposed method learns a frozen action-dependent rule $w_\varphi(s,a)$ that combines the two GAEs from each shared trajectory.

\paragraph{A harmful update with exact rollout probabilities.}
Let two actions return $1$ and $0$, with $q=\pi_{\mathrm{old}}$ assigning each probability $1/2$.  Set $V_1=3/2$, $V_2=-1/2$, and $(\mu_{\mathrm{pilot}},\sigma_{\mathrm{pilot}})=(0,1)$.  Minimum aggregation gives advantages $(-1/2,-3/2)$.  Parameterize the Bernoulli actor by $\pi_\theta(a_+)=1/(1+e^{-\theta})$ and take $\theta=\log(7/3)$, so $\pi_\theta(a_+)=0.7$.  For clipping bounds $[0.8,1.2]$, ratios are $(1.4,0.6)$: the rewarding action remains active because its advantage is negative, while the other action is clipped.  Hence
\begin{equation}
\partial_\theta\mathcal J_{\mathrm{clip}}(\theta\mid s)
=\tfrac12(1.4)(-\tfrac12)(0.3)=-0.105,
\qquad \partial_\theta\E_{\pi_\theta}[G\mid s]=0.21.
\label{eq:harmful-ppo-example}
\end{equation}
The active factors are $v_+=0.42$ and $v_-=0$, so $\mu_{\widetilde A_{\min}}=-1$, $\mu_{v_{\min}}=0.21$, and $\Delta_{\mu,\min}=-0.21$.  The covariance is $0.105$.  The offset overwhelms this positive signal.  The individual conditional offsets are $-1$ and $1$, so weight $1/2$ removes this mean-induced component, but does not undo clipping of the remaining gradient.  This is a constructed example, not a measured result.

\paragraph{High global scores can conceal the local error.}

To see why a global predictive gate need not exclude this example, let a fraction $\omega$ of a held-out set come from this state, with its two actions equally represented.  Split the remaining samples equally between two deterministic-return states with $G=1$ and $G=-1$, where both critics are exact.  Each critic then has mean-squared error $5\omega/4$, while the held-out return variance is $1-\omega/2-\omega^2/4$.  Consequently,
\begin{equation}
R_1^2=R_2^2=1-\frac{5\omega/4}{1-\omega/2-\omega^2/4}
\longrightarrow1\quad\text{as }\omega\to0.
\label{eq:rare-state-predictive-score}
\end{equation}
Both critics can therefore pass any fixed threshold below one while retaining the local harmful update.  This is a gap between global prediction accuracy and state-conditional update quality, not evidence that either gate guarantees improvement on every state.

\newpage
\section{Conditional-Moment Learning and the Drift Bound}
\label{app:conditional-calibration}

The proof follows the construction in Section~\ref{sec:twin-calibration}: the paired advantages define the correction family, the dual learns its conditional mean, and the retention constraint limits displacement from the reference.  We then transfer the remaining mean risk to the actual PPO drift, quantify independent validation, and refine the sensitivity factor.

Fix the rollout law $q$, scoring rule, critic snapshots, common pilot normalizer, and the token occupation measure $d$.  The complete context $c$ precedes the current action and includes the scoring item and reward history.  Expressions conditioned on $s$ below abbreviate conditioning on this $c$.  The weight $w_\varphi(s,a)$ uses the visible prefix and current action.  The auxiliary $h(c)$ may additionally use the fixed scoring context.  Both parameter sets are fixed before the trajectories on which population statements or validation bounds are evaluated.  Assume the stated expectations exist.

\subsection{Action dependence and the limits of twin disagreement}
\label{app:action-dependent-mixture}
Let $\bar w_\varphi(s)=\E_q[w_\varphi\mid s]$ and $\Delta\widetilde A=\widetilde A_1-\widetilde A_2$.  Expanding $\widetilde A_\varphi=\widetilde A_2+w_\varphi\Delta\widetilde A$ gives
\begin{align}
\E_q[\widetilde A_\varphi\mid s]
&=\mu_{\widetilde A_2}+\E_q[w_\varphi\Delta\widetilde A\mid s]\notag\\
&=\bar w_\varphi\mu_{\widetilde A_1}+(1-\bar w_\varphi)\mu_{\widetilde A_2}
+\Cov_q(w_\varphi,\Delta\widetilde A\mid s).
\label{eq:mixture-conditional-mean}
\end{align}
The covariance is zero for a state-only weight, but need not vanish for a frozen action-dependent rule.  Its magnitude is at most
\begin{equation}
|\Cov_q(w_\varphi,\Delta\widetilde A\mid s)|
\leq\sqrt{\Var_q(w_\varphi\mid s)\Var_q(\Delta\widetilde A\mid s)}.
\label{eq:weight-covariance-bound}
\end{equation}
Freezing parameters preserves this dependence.  A rule that takes realized GAEs as input also depends on the sampled continuation, unlike the $w(s,a)$ architecture in Section~\ref{sec:twin-calibration}.

Twin disagreement does not identify a common conditional offset.  Adding the same state-dependent shift $b(s)$ to both normalized advantages leaves $\Delta\widetilde A$ unchanged while increasing both conditional means by $b(s)$.  Relative disagreement alone cannot distinguish these cases.  In the proposed objective, trajectory labels train the separate conditional-mean function.  The pair determines the directions in which the gate can alter that signal.  The dual lemma below holds for any square-integrable signal, so its validity is not evidence that two return critics estimate the mean more accurately than one.  Such a comparison also depends on the correction family, data, model errors, and computation.

In the Monte Carlo case, $A_i=G-V_i(s)$, the critics share the same random return.  Their difference $A_1-A_2=V_2(s)-V_1(s)$ reveals relative prediction differences, not the common value error.  It is state-only, so the covariance in Eq.~\eqref{eq:mixture-conditional-mean} vanishes in this special case even for an action-dependent weight.  The mean uses $\bar w_\varphi(s)$.  General GAE instead includes random later-prefix value differences.  Random initialization alone does not make two advantage outputs independent conditional draws from the rollout law.  Action-dependent mixing can also turn a state baseline into an action-dependent quantity.  With $b_i=V_i+\mu_{\mathrm{pilot}}$,
\begin{equation}
\widetilde A_\varphi=\frac{G-b_\varphi(s,a)}{\sigma_{\mathrm{pilot}}},\qquad
b_\varphi(s,a)=w_\varphi(s,a)b_1(s)+(1-w_\varphi(s,a))b_2(s).
\label{eq:mixture-baseline}
\end{equation}
The usual state-baseline cancellation does not apply to $b_\varphi(s,a)$.  Under an exact, unclipped on-policy score, its contribution contains
$-(b_1-b_2)\E_q[w_\varphi\nabla_\theta\log\pi_\theta\mid s]/\sigma_{\mathrm{pilot}}$.
The method therefore does not claim baseline invariance or oracle policy-gradient recovery.

\paragraph{A centered mixture with nonzero signal.}
\label{app:twin-mixture-example}
The before--after illustration in Eq.~\eqref{eq:complementary-toy} uses equally likely actions $a_+,a_-$.  Let $\widetilde A_1=(0.9,-0.1)$ and $\widetilde A_2=(0.3,-0.7)$, giving the fixed average $\overline A=(0.6,-0.4)$ and conditional mean $0.1$.  Choosing
\begin{equation}
w(a_+)=\tfrac12,\qquad w(a_-)=\tfrac16
\quad\Longrightarrow\quad
\widetilde A_w=(0.6,-0.6),\qquad\mu_{\widetilde A_w}=0
\label{eq:twin-mixture-example-weights}
\end{equation}
keeps both signs unchanged, so the PPO mask and $\mu_v(s)$ are unchanged at fixed policy ratios.  The mean-induced component falls from $0.1\mu_v(s)$ to zero, while the action contrast increases from $1$ to $1.2$.  The signal-retention quantities are $C=\tfrac12(0^2+(-0.2)^2)=0.02$ and $Q=\tfrac12(0.6^2+(-0.4)^2)=0.26$.  Thus $C/Q=1/13$, making this constructed mixture feasible for $\kappa\geq1/13$ within the required range $\kappa<1$.  These are illustrative values, not an experimental setting.

\subsection{The conditional-moment dual}
\label{app:moment-duality}
Recall $R(\varphi)=\E_d[\mu_{\widetilde A_\varphi}(s)^2]$.  For a mean function $h(c)$ with finite $\E_d[h^2]$, the conditional-moment objective is
\begin{equation}
\mathcal L(\varphi,h)=\E_{d,q}[2h(c)\widetilde A_\varphi-h(c)^2].
\label{eq:moment-objective}
\end{equation}
\begin{lemma}[A scalar dual represents the conditional mean square]
\label{lem:moment-duality}
For any square-integrable mixed advantage and $h(c)\in L^2(d)$,
\begin{equation}
\mathcal L(\varphi,h)=R(\varphi)-\|h-\mu_{\widetilde A_\varphi}\|_{L^2(d)}^2.
\label{eq:moment-square-completion}
\end{equation}
Consequently $\sup_{h\in L^2(d)}\mathcal L(\varphi,h)=R(\varphi)$, attained at the conditional mean.
\end{lemma}
\begin{proof}
Since $h$ is measurable before the action, iterated expectation gives
$\E_{d,q}[h(c)\widetilde A_\varphi]=\E_d[h(c)\mu_{\widetilde A_\varphi}(s)]$.
Complete the square in $2h\mu-h^2=\mu^2-(h-\mu)^2$ and average.  Square integrability makes the conditional mean an admissible maximizer.
\end{proof}
Writing $\widetilde A_\varphi=\mu_{\widetilde A_\varphi}+(\widetilde A_\varphi-\mu_{\widetilde A_\varphi})$ and using the zero conditional mean of the second term gives
\begin{equation}
\E_{d,q}[\widetilde A_\varphi^2]
=R(\varphi)+\underbrace{\E_d\Var_q(\widetilde A_\varphi\mid s)}_{\text{within-context variation}}.
\label{eq:sample-square-is-not-mean-square}
\end{equation}
The auxiliary objective isolates a population conditional mean without paired continuations at each state, not an exact mean from one observed advantage.  Fitting one label per context can memorize noise, so population evaluation remains necessary.

For a fixed mixture and $h_\zeta\in\mathcal H\subseteq L^2(d)$, define the restricted optimum $R_{\mathcal H}(\varphi)=\sup_{h\in\mathcal H}\mathcal L(\varphi,h)$, approximation error $\epsilon_{\mathrm{app}}(\varphi)=\inf_{h\in\mathcal H}\E_d[(h-\mu_{\widetilde A_\varphi})^2]$, and inner optimization gap $\epsilon_{\mathrm{opt}}=R_{\mathcal H}(\varphi)-\mathcal L(\varphi,h_\zeta)$.  The dual identity gives $R_{\mathcal H}=R-\epsilon_{\mathrm{app}}$ and the decomposition
\begin{equation}
\epsilon_h:=\E_d[(h_\zeta(c)-\mu_{\widetilde A_\varphi}(s))^2]
=\epsilon_{\mathrm{app}}+\epsilon_{\mathrm{opt}}.
\label{eq:mean-prediction-residual}
\end{equation}
This is error against the conditional mean.  Error against a sampled advantage additionally contains $\E_d\Var_q(\widetilde A_\varphi\mid s)$, so an empirical sample-regression loss is not the same quantity.  The optimization gap concerns the auxiliary head at fixed mixture, not optimization of the gate.

If the auxiliary input is only a compressed feature $z(c)$, the unrestricted optimum on that input is $\E[\widetilde A_\varphi\mid z]$.  The unresolved risk is exactly
\begin{equation}
R(\varphi)-\E_d[(\E_{d,q}[\widetilde A_\varphi\mid z])^2]
=\E_d\Var_d(\mu_{\widetilde A_\varphi}(s)\mid z).
\label{eq:auxiliary-information-gap}
\end{equation}
Conditional means and variances here use the probability law obtained by normalizing $d$.  The outer integrals retain its original mass.  The identity therefore also holds for the finite occupation measure.
This identity follows by conditioning $\mu$ on $z$ and expanding its second moment.  In particular, omitting the scoring item can hide item-specific offsets rather than eliminate them.

\subsection{One-rollout stochastic gradients}
\label{app:single-rollout-gradients}
Keep $q$, critic features, and GAE endpoints detached during head fitting.  Substituting the mixture into the objective gives
\begin{equation}
\mathcal L(\varphi,h)
=\E_{d,q}[2h(c)\widetilde A_2-h(c)^2]
+2\E_{d,q}[h(c)w_\varphi(s,a)\Delta\widetilde A].
\label{eq:twin-role-in-moment}
\end{equation}
For differentiable heads, let $\ell_{\mathrm{mom}}=2h_\zeta(c)\widetilde A_\varphi-h_\zeta(c)^2$.  Its per-sample gradients are
\begin{equation}
\begin{aligned}
\nabla_\varphi\ell_{\mathrm{mom}}&=2h_\zeta(c)\Delta\widetilde A\,\nabla_\varphi w_\varphi(s,a),\\
\nabla_\zeta\ell_{\mathrm{mom}}&=2(\widetilde A_\varphi-h_\zeta(c))\nabla_\zeta h_\zeta(c).
\end{aligned}
\label{eq:main-head-gradients}
\end{equation}
The predicted offset sets the correction direction, and the advantage difference determines the available adjustment.  If the advantages coincide, the weight gradient vanishes even when the predicted mean is nonzero.  Giving $h$ the sampled action would change its target to an action-conditional mean and could suppress useful action differences.  Assume these gradients and the constraint gradient are integrable and differentiation can be interchanged with expectation.  A sufficient local condition is a uniformly bounded gate Jacobian and a common $L^2$ envelope for the auxiliary output and its Jacobian.  The moment gradients are then unbiased for $\mathcal L$ at fixed parameters, while their relation to $R$ depends on mean-approximation quality.  Under these conditions,
\begin{align}
\nabla_\varphi\mathcal L
&=\E_{d,q}[\nabla_\varphi\ell_{\mathrm{mom}}],\\
\nabla_\zeta\mathcal L
&=\E_{d,q}[\nabla_\zeta\ell_{\mathrm{mom}}],\\
\nabla_\varphi C
&=\E_{d,q}[2(\widetilde A_\varphi-\overline A)\Delta\widetilde A\,
\nabla_\varphi w_\varphi(s,a)].
\label{eq:single-rollout-gradients}
\end{align}
One complete trajectory gives unbiased trajectory-summed estimates of these expressions with the fixed aggregation in Appendix~\ref{app:restart-contract}.  Tokens on a trajectory may be correlated, since linearity of expectation is sufficient for this statement.  Independence is needed across fresh trajectories for the validation concentration bound below.  Conditional on training history, an online update uses parameters fixed before its fresh trajectory is drawn.  Reusing a fixed trajectory set gives stochastic gradients of its empirical objective.  Evaluating the selected candidate on that same set does not provide independent population validation.  This result concerns a stochastic observation, not the total number of trajectories needed to fit $h$ or reduce $R$ to a specified tolerance.  Appendix~\ref{app:rollout-accounting} accounts for every sampling phase.

For a gate function $w$, write $R(w),C(w),\mathcal L(w,h)$ for the same objectives evaluated at that function.  Introducing the nonnegative retention multiplier $\beta$ gives
\begin{equation}
\mathcal J(w,h,\beta)=\mathcal L(w,h)+\beta\{C(w)-\kappa Q\}.
\label{eq:moment-lagrangian}
\end{equation}
A practical update evaluates Eq.~\eqref{eq:moment-lagrangian} at $w=w_\varphi$ and $h=h_\zeta$, descending in $\varphi$ and ascending in $\zeta,\beta$.  Its per-token integrand is
\begin{equation}
j=2h_\zeta(c)\widetilde A_\varphi-h_\zeta(c)^2
+\beta\{(w_\varphi-\tfrac12)^2(\Delta\widetilde A)^2-\kappa\overline A^2\}.
\label{eq:single-rollout-saddle-integrand}
\end{equation}
The weight gradient adds $\beta\nabla_\varphi C$ to the moment gradient, and the multiplier derivative is $(w_\varphi-\tfrac12)^2(\Delta\widetilde A)^2-\kappa\overline A^2$.  All products use the two advantages from the same trajectory.  No independent continuation is needed to form a product or a prefix-wise Monte Carlo mean.  Projected ascent keeps $\beta$ nonnegative.  Proposition~\ref{prop:population-saddle} characterizes the ideal function-space problem, not global convergence of neural alternating optimization.  After fitting, freeze the candidate and use independent validation and actor trajectories.  Detach the complete actor advantage so no actor gradient enters the heads, critics, or normalizer.

\subsection{Population saddle points and optimal weights}
\label{app:population-saddle}
Let $x=(s,a)$ denote only the visible gate input.  Conditioning on $x$ here does not invoke the shorthand that also fixes the hidden parts of $c$.  Conditional expectations use the normalized joint law of $(d,q)$, while outer integrals retain its finite nonzero mass $M$.  Let $\mathcal W$ contain all measurable $[0,1]$-valued functions of $x$, and write
$\widetilde A_w=\overline A+(w(x)-\tfrac12)\Delta\widetilde A$.
The feasible set is $\mathcal F=\{w\in\mathcal W:C(w)\leq\kappa Q\}$, and $R^*=\min_{w\in\mathcal F}R(w)$.  The mixture is affine in $w$, so $\mathcal L$ is affine in $w$ and concave in $h$.  The quadratic $C$ is convex, and the multiplier term is affine in $\beta$.

Under the assumptions of Proposition~\ref{prop:population-saddle}, $\mathcal J$ is convex in $w$ and jointly concave in $(h,\beta)$.  We establish a saddle point with value $R^*$ and derive the optimal gate in Eq.~\eqref{eq:optimal-saddle-weight}.
\begin{proof}[Proof of Proposition~\ref{prop:population-saddle}]
\emph{Existence of a feasible optimum.}
Set $a(x)=\E_{d,q}[(\Delta\widetilde A)^2\mid x]$ and define $d\nu(x)=M a(x)dP_x(x)$, where $P_x$ is the marginal of the normalized joint law.  Work in the Hilbert space $L^2(\nu)$.  Gates that differ only where $a=0$ have the same effect on the signal.  The set $\mathcal W$ is nonempty, closed, bounded, and convex in this space, hence weakly compact.  The linear operator
$Tu=\E_{d,q}[u(x)\Delta\widetilde A\mid c]$
obeys, by conditional Jensen,
\begin{equation}
\|Tu\|_{L^2(d)}^2\leq\E_{d,q}[u(x)^2(\Delta\widetilde A)^2]
=\|u\|_{L^2(\nu)}^2.
\label{eq:saddle-gate-operator}
\end{equation}
Thus $R(w)=\|\mu_{\overline A}+T(w-\tfrac12)\|_{L^2(d)}^2$ and $C(w)=\|w-\tfrac12\|_{L^2(\nu)}^2$ are continuous convex functions and are weakly lower semicontinuous.  The feasible set is a nonempty weakly compact subset of $\mathcal W$, so it contains a minimizer $w^*$.

\emph{Strong duality and the saddle.}
The reference satisfies $C(1/2)=0<\kappa Q$, giving Slater's condition relative to $\mathcal W$.  Convex Lagrange duality therefore supplies a finite multiplier $\beta^*\geq0$ such that $w^*$ minimizes $R+\beta^*C$ over $\mathcal W$ and
\begin{equation}
C(w^*)\leq\kappa Q,\qquad
\beta^*\{C(w^*)-\kappa Q\}=0.
\label{eq:saddle-complementarity}
\end{equation}
Take $h^*=\E_q[\widetilde A_{w^*}\mid c]$, the maximizing response from Lemma~\ref{lem:moment-duality}.  For $u=w-w^*$, the directional derivative of $\mathcal L(\cdot,h^*)$ at $w^*$ equals that of $R$.  The first-order condition for $R+\beta^*C$, followed by expanding the quadratic constraint, gives
\begin{align}
\mathcal J(w,h^*,\beta^*)-\mathcal J(w^*,h^*,\beta^*)
&=2\E_{d,q}[h^*u\Delta\widetilde A]
+2\beta^*\E_{d,q}[(w^*-\tfrac12)u(\Delta\widetilde A)^2]\notag\\
&\quad+\beta^*\|u\|_{L^2(\nu)}^2\geq0.
\label{eq:saddle-first-order}
\end{align}
The square-completion identity and complementary slackness give the other side.  For all admissible $w,h,\beta$,
\begin{equation}
\mathcal J(w^*,h,\beta)\leq\mathcal J(w^*,h^*,\beta^*)=R^*
\leq\mathcal J(w,h^*,\beta^*).
\label{eq:population-saddle-inequality}
\end{equation}
Consequently the unrestricted constrained problem has the strong-duality representation
\begin{equation}
\min_{w\in\mathcal W}\sup_{h\in L^2(d),\,\beta\geq0}\mathcal J(w,h,\beta)
=\max_{h\in L^2(d),\,\beta\geq0}\min_{w\in\mathcal W}\mathcal J(w,h,\beta)
=R^*.
\label{eq:population-strong-duality}
\end{equation}
For an infeasible gate the inner supremum is infinite through $\beta$, which is why the left side retains $\sup$.

\emph{Optimal visible-input weights.}
Define $b_h(x)=\E_{d,q}[h(c)\Delta\widetilde A\mid x]$.  Given $h,\beta$, the terms depending on $w(x)$ are
\begin{equation}
2b_h(x)(w(x)-\tfrac12)+\beta a(x)(w(x)-\tfrac12)^2.
\label{eq:gate-quadratic-response}
\end{equation}
For $\beta>0$ and $a(x)>0$, its unconstrained derivative vanishes at $w=1/2-b_h/(\beta a)$.  Projecting onto $[0,1]$ and substituting $h^*,\beta^*$ proves Eq.~\eqref{eq:optimal-saddle-weight}.  Conditional Cauchy--Schwarz gives $|b_h|^2\leq a\E[h^2\mid x]$, so the ratio is well-defined on $a>0$, and $b_h=0$ on $a=0$.  The construction averages over scoring information and future randomness given visible $x$, so it respects the gate's input restriction.
\end{proof}

\paragraph{Boundary cases and interpretation.}
If $a(x)=0$, both critics give identical advantages conditional on $x$ almost surely, and any weight there is equivalent.  If $\beta^*=0$, minimize the linear term in Eq.~\eqref{eq:gate-quadratic-response}: $b_{h^*}>0$ requires $w^*=0$, $b_{h^*}<0$ requires $w^*=1$, and $b_{h^*}=0$ permits any weight consistent with the remaining optimality conditions.  An active constraint need not have a positive multiplier.  The relations are self-consistent because $h^*$ depends on $w^*$ and $\beta^*$ satisfies Eq.~\eqref{eq:saddle-complementarity}.  They do not prescribe substituting a realized GAE into the gate at actor-update time.

When $\kappa=0$, feasibility fixes the mixed signal to $\overline A$, but strict feasibility fails and a finite saddle multiplier need not exist.  The original signal-retention bound still applies.  The case $Q=0$ is treated in Appendix~\ref{app:signal-retention}.  Neural parameterizations need not inherit convexity in $\varphi,\zeta$, and empirical fitting remains distinct from the population problem.  In particular, an unrestricted auxiliary function could memorize a single label at each training context, turning the empirical inner optimum into a sample-square loss.  Shared representations and independent evaluation remain important.  None of the optimality analysis adds a group-sampling step to the algorithm.

\subsection{Signal retention and attainable risk}
\label{app:signal-retention}
\begin{proposition}[The displacement constraint excludes a zero signal]
\label{prop:signal-retention}
Let $Q=\|\overline A\|_{L^2(d,q)}^2>0$ and $C(\varphi)\leq\kappa Q$, $0\leq\kappa<1$.  Then
\begin{equation}
\|\widetilde A_\varphi\|_{L^2}\geq(1-\sqrt\kappa)\sqrt Q,\qquad
\E[\widetilde A_\varphi\overline A]\geq(1-\sqrt\kappa)Q.
\label{eq:signal-retention-guarantee}
\end{equation}
\end{proposition}
\begin{proof}
Write $e=\widetilde A_\varphi-\overline A$.  Using the common $(d,q)$ weighting, the reverse triangle inequality gives the full chain
\begin{equation}
\|\widetilde A_\varphi\|_{L^2}\geq\|\overline A\|_{L^2}-\|\widetilde A_\varphi-\overline A\|_{L^2}=\sqrt Q-\sqrt{C(\varphi)}\geq(1-\sqrt\kappa)\sqrt Q>0.
\label{eq:retention-triangle}
\end{equation}
Cauchy--Schwarz gives
$\E[\widetilde A_\varphi\overline A]=Q+\E[e\overline A]\geq Q-\sqrt{C Q}$.
Insert $C\leq\kappa Q$ to obtain the alignment bound.
\end{proof}
These bounds preserve aggregate energy and alignment with the fixed twin mean, which need not be an oracle advantage.  They do not preserve every state's action ordering, conditional covariance, or gradient direction.  When $Q=0$, the reference itself has no signal and feasibility forces a zero mixture almost surely.  The window provides no policy-learning signal to retain.  This case is diagnosed rather than interpreted as successful drift correction.

The constant gate $w=1/2$ makes $C=0$.  Candidate weights are checked on independent complete trajectories.  The base protocol falls back to this constant gate after a failed empirical check.  That fallback guarantees feasibility, not improved drift.  If population enforcement is claimed, a valid upper bound on $C-\kappa Q$ is required.  Selecting or tuning $\kappa$ on the same check invalidates a fixed-candidate confidence interpretation.

\begin{proposition}[A population comparison with the fixed mean]
\label{prop:population-mixture-comparison}
Assume the gate family contains $w=1/2$, denoted $\varphi_0$.  A feasible global minimizer of the unrestricted dual satisfies $R(\varphi^*)\leq R(\varphi_0)$.  More generally, if
$R_{\mathcal H}(\hat\varphi)\leq R_{\mathcal H}(\varphi_0)+\epsilon_{\mathrm{outer}}$, then
\begin{equation}
R(\hat\varphi)\leq R(\varphi_0)+\epsilon_{\mathrm{app}}(\hat\varphi)
-\epsilon_{\mathrm{app}}(\varphi_0)+\epsilon_{\mathrm{outer}}.
\label{eq:population-mixture-comparison}
\end{equation}
\end{proposition}
\begin{proof}
The reference is feasible.  For the second statement substitute
$R_{\mathcal H}(\varphi)=R(\varphi)-\epsilon_{\mathrm{app}}(\varphi)$
on both sides of the assumed population objective comparison.
\end{proof}
This is a conditional comparison, not a claim that stochastic head training attains its premise.  A nonzero feasible optimum is allowed.  Shared critic bias, restricted inputs, and the retention constraint can all prevent exact centering.

\subsection{Actual PPO gradient and error bounds}
\label{app:drift-proof}
For the detached mixed advantage, use its own active mask
\begin{equation}
I_\varphi=\one\{\widetilde A_\varphi>0,\rho<\rho_+\}
+\one\{\widetilde A_\varphi<0,\rho>\rho_-\}+\one\{\widetilde A_\varphi=0\},
\qquad v_\varphi=\rho I_\varphi\nabla_\theta\log\pi_\theta.
\label{eq:ppo-active-mask}
\end{equation}
At clipping boundaries choose a fixed subgradient coefficient in $[0,1]$.  The zero-advantage convention does not change its zero gradient contribution.
\paragraph{Clipping after mixing.}
\label{app:regularization-identity}
For any scalar normalized advantage $x$, the clipped surrogate obeys the pointwise identity
\begin{equation}
\min\{\rho x,\clip(\rho,\rho_-,\rho_+)x\}
=\rho x-x_+(\rho-\rho_+)_+-(-x)_+(\rho_{-}-\rho)_+.
\label{eq:pointwise-regularization}
\end{equation}
The identity follows by checking the two signs of $x$ and the clipping thresholds.  It does not require a state-only weight.  It also shows why clipping the two advantages separately and then mixing their gradients need not equal clipping their mixture.  The conditional-moment objective trains the weight before actor sampling.  It is not an additional differentiable penalty inside the actor loss.

For the actual mixed signal and its own mask,
\begin{equation}
g_\varphi(s)=\Cov_q(\widetilde A_\varphi,v_\varphi\mid s)
+\underbrace{\mu_{\widetilde A_\varphi}(s)\mu_{v_\varphi}(s)}_{\Delta_{\mu,\varphi}(s)}.
\label{eq:weighted-drift-decomposition}
\end{equation}
Expand the definition of conditional covariance.  This step does not require independence between the weight and action or between the two factors.  Detaching the advantage ensures that this is the score-form actor gradient, without derivatives through the heads.

\begin{proof}[Proof of Theorem~\ref{thm:complementary-drift}]
Using the actual mixed advantage and its own PPO mask, Cauchy--Schwarz gives
\begin{align}
\E_d\|\Delta_{\mu,\varphi}\|_2
&=\E_d[|\mu_{\widetilde A_\varphi}|\,\|\mu_{v_\varphi}\|_2]
\label{eq:drift-cauchy-schwarz}\\
&\leq\sqrt{R(\varphi)}\sqrt{\E_d\|\mu_{v_\varphi}\|_2^2}
\leq B\sqrt{R(\varphi)}.
\label{eq:conditional-drift-bound}
\end{align}
Lemma~\ref{lem:moment-duality} and Eq.~\eqref{eq:mean-prediction-residual} yield
\begin{equation}
R(\varphi)=\mathcal L(\varphi,h_\zeta)+\epsilon_h
=\mathcal L(\varphi,h_\zeta)+\epsilon_{\mathrm{app}}+\epsilon_{\mathrm{opt}}.
\label{eq:restricted-dual-gap}
\end{equation}
Insert the independent validation upper bound on $\mathcal L$.  This gives the theorem with $\epsilon_h$, or equivalently with its approximation and inner-optimization components.  On that event the sum under the square root is nonnegative, since it bounds $R\geq0$.
\end{proof}
For all convex mixtures, a sufficient common factor is
$B=C_\rho=(\E_{d,q}[\rho^2\|\nabla_\theta\log\pi_\theta\|_2^2])^{1/2}$,
by Jensen and $0\leq I_\varphi\leq1$.  This factor must cover every actor parameter if the bound is claimed throughout PPO reuse.  A reduced risk improves this common upper bound.  It does not imply monotone decrease of the actual $\Delta_{\mu,\varphi}$ when its multiplier changes, or of full-gradient error.  The clipping nonexpansiveness and GAE error bounds in Appendix~\ref{app:value-gradient-proof} remain separate comparisons with an oracle signal.

\subsection{Independent trajectory validation and its limits}
\label{app:moment-validation}
For a candidate $(\varphi,\zeta)$ fixed independently of $n$ mutually independent validation trajectories drawn from the declared item and rollout law, use the trajectory statistic
\begin{equation}
Z_j=\frac1{T_{\max}}\sum_{t=1}^{T_{\max}}m_{jt}
\{2h_\zeta(c_{jt})\widetilde A_{\varphi,jt}-h_\zeta(c_{jt})^2\},
\qquad\widehat{\mathcal L}_{\mathrm{val}}=\frac1n\sum_jZ_j,
\label{eq:trajectory-moment-validation}
\end{equation}
where $m_{jt}$ marks a valid pre-action token, as in Appendix~\ref{app:restart-contract}.  If $|\widetilde A_i|\leq K$ and $|h_\zeta|\leq K_h$ are known population bounds, then $|Z_j|\leq C_L:=2K_hK+K_h^2$.  Hoeffding's inequality gives, with probability at least $1-\delta$,
\begin{equation}
\mathcal L(\varphi,h_\zeta)\leq\widehat{\mathcal L}_{\mathrm{val}}
+C_L\sqrt{\frac{2\log(1/\delta)}{n}}.
\label{eq:moment-statistical-bound}
\end{equation}
The independent units are trajectories, not their tokens.  For the retention check, $0\leq(\widetilde A_\varphi-\overline A)^2\leq K^2$ and $0\leq\overline A^2\leq K^2$.  An upper confidence bound on $C-\kappa Q$ is its empirical trajectory average plus
$(1+\kappa)K^2\sqrt{\log(1/\delta)/(2n)}$.
Independent evaluation must follow candidate selection.  Repeated checks require fresh trajectory sets and a declared error allocation.

\begin{corollary}[What a drift certificate would require]
\label{cor:admission-budget}
If valid upper bounds $a,o$ on $\epsilon_{\mathrm{app}},\epsilon_{\mathrm{opt}}$ are also available, define
\begin{equation}
U_R=\max\{0,\widehat{\mathcal L}_{\mathrm{val}}+\epsilon_{\mathrm{stat}}+a+o\}.
\qquad U_R\leq\tau^2\ \Longrightarrow\ \E_d\|\Delta_{\mu,\varphi}\|_2\leq B\tau.
\label{eq:admitted-drift-certificate}
\end{equation}
\end{corollary}
The result is conditional on all these bounds.  Concentration of $Z_j$ alone does not upper-bound $R$, because a restricted or incompletely fitted auxiliary head provides a lower objective.  A low empirical saddle value, a small empirical head gradient, or twin disagreement supplies neither $a$ nor $o$.  The base algorithm therefore uses empirical signal retention and reports the moment diagnostics without claiming this certificate.  Bounded rewards do not automatically bound neural values or normalized GAE.  Clipping only validation outputs would certify a different signal.

\subsection{Policy sensitivity, clipping, and sampling mismatch}
\label{app:fisher-drift}
The coefficient $B$ measures amplification by the actor and its active clipping branches.  Fix one complete pre-action context and its frozen selected mixture, and write $\mu_{\widetilde A},\mu_v,\Delta_\mu,I$ for the corresponding weighted quantities in this subsection.  Let $\xi$ denote randomness following the current action.  Introduce
\begin{equation}
\nu_s(a,\xi)=\pi_\theta(a\mid s)q(\xi\mid s,a),
\qquad \omega(a,s)=\frac{q(a\mid s)}{\pi_{\mathrm{old}}(a\mid s)},
\label{eq:current-action-law}
\end{equation}
where $\omega$ is a sampling-mismatch ratio, not the mixture weight $w$.  Only the current action is reweighted.  The continuation law and both GAEs stay fixed.  Assume defined continuations and positive recorded probabilities on actor support, differentiable parameter-independent actor support, and finite $F(s)=\E_{\pi_\theta}\|u\|_2^2$ and $\Var_{\nu_s}(\omega I)$, where $u=\nabla_\theta\log\pi_\theta(a\mid s)$.
\begin{proposition}[Amplification by policy sensitivity and uneven weights]
\label{prop:fisher-amplification}
Under these fixed-context conditions,
\begin{equation}
\|\Delta_\mu(s)\|_2\leq|\mu_{\widetilde A}(s)|\sqrt{F(s)\Var_{\nu_s}(\omega I)}.
\label{eq:mismatch-fisher-bound}
\end{equation}
\end{proposition}
\begin{proof}
Expand the expectation by first sampling the action and then its continuation:
\begin{align}
\mu_v(s)&=\sum_a q(a\mid s)\frac{\pi_\theta(a\mid s)}{\pi_{\mathrm{old}}(a\mid s)}
\E_{\xi\sim q(\cdot\mid s,a)}[Iu]\\
&=\sum_a\pi_\theta(a\mid s)\omega(a,s)\E_{\xi\sim q(\cdot\mid s,a)}[Iu]
=\E_{\nu_s}[\omega Iu].
\label{eq:mask-score-covariance}
\end{align}
This changes only the action weighting.  Since $u$ depends on the current action, not on its future continuation, $\E_{\nu_s}u=\sum_a\nabla_\theta\pi_\theta(a\mid s)=0$.  Let $z=\omega I$ and $\bar z=\E_{\nu_s}z$.  Subtracting this constant weight therefore leaves $\mu_v$ unchanged:
\begin{equation}
\mu_v=\E_{\nu_s}[(z-\bar z)u],\qquad
\|\mu_v\|_2\leq\E_{\nu_s}[|z-\bar z|\|u\|_2]
\leq\sqrt{\Var_{\nu_s}(z)F(s)}.
\end{equation}
The last step is Cauchy--Schwarz.  Multiplying by $|\mu_{\widetilde A}|$ proves the proposition.  The offset $\mu_{\widetilde A}$ itself is still evaluated under $q$.  It was not redefined by this change of measure.
\end{proof}

For exact recorded probabilities on full actor support, $\omega=1$.  Away from clipping boundaries, $I$ is binary.  Write $p_c=\Pr_{\nu_s}(I=0)$.
\begin{corollary}[Selective clipping is the exact-ratio multiplier]
\label{cor:fisher-clipping}
With exact recorded probabilities and the binary mask above,
\begin{equation}
\|\Delta_\mu(s)\|_2\leq|\mu_{\widetilde A}(s)|\sqrt{F(s)p_c(s)(1-p_c(s))}.
\label{eq:main-fisher-factor}
\end{equation}
\end{corollary}
\begin{proof}
A binary variable with mean $1-p_c$ has second moment $1-p_c$, so its variance is $(1-p_c)-(1-p_c)^2=p_c(1-p_c)$.  Substitute $\omega=1$ into Proposition~\ref{prop:fisher-amplification}.
\end{proof}
Applying Cauchy--Schwarz across the occupation measure also gives
\begin{equation}
\E_d\|\Delta_\mu\|_2\leq\sqrt{R(\varphi)\E_d[Fp_c(1-p_c)]}.
\label{eq:fisher-drift-bound}
\end{equation}
The three factors are offset size, policy sensitivity, and clipping imbalance.  All branches active means score cancellation.  All inactive means no update.  Selectively removing directions can break the balance.  At a boundary with $I\in[0,1]$, use $p_c=1-\E_{\nu_s}I$ and $\Var(I)\leq\E I-(\E I)^2=p_c(1-p_c)$.  Here $p_c$ is a weight fraction rather than an event probability.  In either case it uses current-action weighting, not the ordinary rollout clip fraction.

\paragraph{Tightness in the two-action example.}
At $\pi_\theta(a_+)=0.7$, the Bernoulli scores are $u_+=0.3$ and $u_-=-0.7$.  Only the positive-reward action remains active in Eq.~\eqref{eq:harmful-ppo-example}, so
\begin{equation}
F=0.7(0.3)^2+0.3(0.7)^2=0.21,
\quad p_c=0.3,
\quad |\Delta_\mu|=0.21=|\mu_{\widetilde A}|\sqrt{Fp_c(1-p_c)}.
\label{eq:fisher-toy-tightness}
\end{equation}
The example reaches equality because the surviving score direction aligns with the weight imbalance.  Calibration changes the offset and can also change which branches survive.  Both sides of the bound must use that selected mixture.  Large negative-advantage ratios remain active, so clipping alone supplies no uniform bound on $C_\rho$ or the mismatch factor.  A guarantee over the entire update window requires a common bound on these factors.  With no clipping and exact ratios, score cancellation already gives $\mu_v=0$.

\newpage
\section{Scope of Conditional Calibration}
\label{app:calibration-scope}
The preceding results concern the mean-induced term of the actual mixed GAE within a frozen window.  We now examine what changes when the target is replaced by TD, scoring information is withheld, or the prefix distribution moves, before distinguishing neighboring actor--critic methods.  These comparisons specify which assumptions must remain fixed for the guarantees to apply.

\subsection{Why one-step TD is not substituted for GAE}
\label{app:td-moment}
Let $b_i(s)=\E_q[\delta_{i,t}\mid s_t=s]$ for a frozen critic.  A single TD residual is an unbiased observation of
\begin{equation}
b_i(s)=(T^qV_i-V_i)(s),
\end{equation}
where $T^q$ includes reward and terminal masking.  This is not generally the conditional mean of the GAE used by the actor.  With a fixed trace parameter $\lambda$, a Markov context including time and reward history, and zero terminal bootstrap, define the raw mean $m_i(s)=\E_q[A_{i,t}\mid s_t=s]$.  Iterated expectation gives
\begin{equation}
m_i(s)=\E_q[\delta_{i,t}+\gamma\lambda(1-d_t)m_i(s_{t+1})\mid s_t=s],
\qquad\mu_{\widetilde A_i}(s)=\frac{m_i(s)-\mu_{\mathrm{pilot}}}{\sigma_{\mathrm{pilot}}}.
\label{eq:td-mean-bellman}
\end{equation}
Replacing $m_i$ by a learned bootstrap gives an unbiased sample of that model's Bellman target, not of the true $m_i$ until its continuation mean is correct.  For a fixed $w(s,a)$, the corresponding raw mixed mean is
\begin{equation}
\E_q\!\left[w\delta_1+(1-w)\delta_2+
\gamma\lambda(1-d_t)\{w m_1(s_{t+1})+(1-w)m_2(s_{t+1})\}\mid s\right].
\label{eq:td-action-weighted-mean}
\end{equation}
The next-state terms retain the current weight.  Substituting the next state's mixed mean is generally incorrect.  The actual implementation uses $\lambda(L)$ selected by realized trajectory length.  It can correlate with future residuals and cannot be pulled out of the conditional expectation to obtain Eq.~\eqref{eq:td-mean-bellman} unchanged.  The moment objective instead uses the actual sampled GAE and is valid with that length dependence.  TD-based mean learning is a separate fixed-trace variant, not an unbiased shortcut asserted for the present algorithm.

\subsection{Training-time scoring information and visible-state limits}
\label{app:information-limit}
Let $W$ denote the future-text scoring context and restore $V^q(s,W)=\E_q[G\mid s,W]$.  Assume square-integrable conditional returns, a fixed $\sigma_{\mathrm{pilot}}>0$, and a square-integrable baseline $b(s)$ restricted to visible state.
\begin{lemma}[Hidden scoring text leaves a variance floor]
\label{lem:visible-floor}
Under these information and moment conditions,
\begin{equation}
\frac{\E_{s,W}[(V^q(s,W)-b(s))^2]}{\sigma_{\mathrm{pilot}}^2}
\geq\frac{\E_s\Var_{W\mid s}(V^q(s,W))}{\sigma_{\mathrm{pilot}}^2}.
\label{eq:visible-state-floor}
\end{equation}
\end{lemma}
\begin{proof}
Let $\bar v(s)=\E_{W\mid s}V^q(s,W)$ and write $V^q-b=(V^q-\bar v)+(\bar v-b)$.  After squaring, the cross term has conditional mean zero since $\E_{W\mid s}(V^q-\bar v)=0$.  The remaining terms are $\Var_{W\mid s}(V^q)+(\bar v-b)^2$.  Drop the nonnegative square, average over $s$, and divide by $\sigma_{\mathrm{pilot}}^2$.
\end{proof}
This floor applies directly to a state-only Monte Carlo baseline.  An action-dependent $b_\varphi(s,a)$ is a different object and must be analyzed through the actual mixed conditional mean.  The weight in Section~\ref{sec:twin-calibration} sees visible prefix/action features.  The auxiliary mean function receives the complete conditioning information, including $W$, through a training-only input path.  Neither the actor nor the return critics receive $W$.  Restricting the auxiliary input to visible features instead targets an average over scoring items and leaves the information gap in Eq.~\eqref{eq:auxiliary-information-gap}.  Even full context does not ensure representability by a finite auxiliary model, or existence of a feasible gate with zero offset.

\subsection{Why qualification must follow the prefix distribution}
\label{app:prefix-transfer}
The learned gate defines a frozen rule on new prefixes, but its fitted occupation risk need not transfer to a different prefix distribution.  The following supplementary result bounds this distribution effect while keeping the advantage-generating continuation law fixed.  The occupancy argument follows trust-region analysis \citep{schulman2015trpo}.  Consider fixed horizon $T$, identical initial text distributions and transitions, and policies $q$ and $\pi$.  Let $d_t^q,d_t^\pi$ denote pre-action distributions, including the scoring item as an analytical context component.  Keep both heads, critics, scorer, normalizer, and the continuation law defining $\mu_{\widetilde A_t}$ fixed.  Assume this function is defined on both reachable supports and $|\mu_{\widetilde A_t}(s)|\leq K$.

Define the time-averaged risk and the local action-distribution change by
\begin{equation}
\mathcal L_{d^\pi}(\mu_{\widetilde A})=\frac1T\sum_{t=1}^T\E_{d_t^\pi}[\mu_{\widetilde A_t}(s)^2],
\qquad
\delta_j=\E_{d_j^q}D_{\mathrm{TV}}\!\left(q(\cdot\mid s),\pi(\cdot\mid s)\right),
\label{eq:prefix-risk-definition}
\end{equation}
where $D_{\mathrm{TV}}(p,q)=\tfrac12\|p-q\|_1$.  It measures how much probability mass must be redistributed to turn one distribution into the other.
\begin{proposition}[Early action changes affect more later prefixes]
\label{prop:prefix-transfer}
Under the fixed-horizon and frozen-score conditions above,
\begin{equation}
\mathcal L_{d^\pi}(\mu_{\widetilde A})\leq\mathcal L_{d^q}(\mu_{\widetilde A})
+\frac{K^2}{T}\sum_{j=1}^{T-1}(T-j)\delta_j.
\label{eq:prefix-calibration-transfer}
\end{equation}
\end{proposition}
\begin{proof}
\emph{Step 1: separate old state differences from new action differences.}  Let $P_q,P_\pi$ be the next-state transition kernels.  Add and subtract the distribution obtained by applying $\pi$ to the old states:
\begin{equation}
d_{t+1}^\pi-d_{t+1}^q=(d_t^\pi-d_t^q)P_\pi+d_t^q(P_\pi-P_q).
\end{equation}
A common stochastic kernel cannot increase total variation: the triangle inequality and row sums of one give $\|(p-q)P\|_1\leq\sum_x|p(x)-q(x)|\sum_yP(y\mid x)=\|p-q\|_1$.  Applying the same argument to the action-to-state transition bounds the second term by the average action difference $\delta_t$.  Therefore
\begin{equation}
D_{\mathrm{TV}}(d_{t+1}^\pi,d_{t+1}^q)
\leq D_{\mathrm{TV}}(d_t^\pi,d_t^q)+\delta_t
\leq\sum_{j=1}^t\delta_j,
\end{equation}
since both policies start from the same initial distribution.

\emph{Step 2: turn state shift into score shift.}  For any function $f$ in $[0,K^2]$, only the positive part of $p-q$ can increase its expectation, so $\E_pf-\E_qf\leq K^2\sum_{x:p(x)>q(x)}(p(x)-q(x))=K^2D_{\mathrm{TV}}(p,q)$.  Apply this to $f=\mu_{\widetilde A_t}^2$.

\emph{Step 3: count how many times each action difference appears.}  Averaging over $t$ gives a double sum.  The change at position $j$ appears for $t=j+1,\ldots,T$, exactly $T-j$ times:
\begin{equation}
\frac1T\sum_{t=1}^T K^2\sum_{j<t}\delta_j
=\frac{K^2}{T}\sum_{j=1}^{T-1}(T-j)\delta_j.
\end{equation}
This is the claimed bound.
\end{proof}
An early action can redirect more of the suffix, whereas the final action changes no later pre-action state.  This counts possible propagation of state shift, not the reward importance of the last action.  Pinsker's inequality also permits
\begin{equation}
\delta_j\leq\sqrt{\tfrac12\E_{d_j^q}
D_{\mathrm{KL}}\!\left(q(\cdot\mid s)\,\|\,\pi(\cdot\mid s)\right)}.
\label{eq:prefix-kl-control}
\end{equation}
This relates bounded policy reuse to calibration coverage without changing the PPO objective.  Both risks in Eq.~\eqref{eq:prefix-calibration-transfer} evaluate the same frozen $\mu_{\widetilde A_t}$, which still uses $q$ continuations.  Updating the continuation law, scorer, or normalizers changes that function and requires fresh validation.  The displayed weights use fixed-length, uniform-time averaging.  Variable-length valid-token normalization needs the corresponding occupancy weights rather than this formula unchanged.

\subsection{Relation to Neighboring Actor--Critic Methods}
\label{app:algorithmic-boundaries}

\paragraph{SAO.}
Single-Rollout Asynchronous Optimization targets asynchronous agentic post-training, where each prompt produces one visible action trajectory and the learner may update while newer trajectories are still being generated \citep{hou2026sao}.  Its central problems are policy lag and the variance created by removing group-relative rollouts.  SAO therefore forms token-level importance ratios against the rollout policy, masks ratios outside a double-sided interval, trains a single value model more frequently than the actor, freezes selected value-model parameters, and skips externally supplied observation tokens when propagating GAE.

\tfour{} operates on unlabeled text: hidden thought tokens are actions and teacher-forced continuation likelihood defines reward.  There are no external observation steps to skip.  Both methods use learned values and a single trajectory per selected item.  Their learning targets and data schedules differ.  \tfour{} separately qualifies two return critics, learns an action-dependent advantage mixture through a conditional-moment saddle objective, and freezes it for fresh actor trajectories.  The scalar auxiliary mean function is distinct from a return critic.  SAO and \tfour{} both use the length-adaptive GAE trace from VAPO \citep{yue2025vapo}.  That component is not claimed as new.

\paragraph{TD3 and SAC.}
TD3 and SAC address continuous-control, off-policy actor--critic learning with replay buffers and action-value functions \citep{fujimoto2018addressing,haarnoja2018soft}.  TD3 reduces maximization bias with a minimum over bootstrapped Q targets and delays policy updates.  SAC optimizes an entropy-regularized return through soft Bellman backups.  \tfour{} instead uses locally on-policy collection with bounded within-iteration reuse, state-value critics, observed shaped returns, and a clipped likelihood-ratio objective over discrete latent tokens.  Its action-dependent convex weights are learned through a conditional-moment objective, not a minimum over Bellman targets.  The shared motivation is to manage function-approximation error.  The mechanisms and data requirements differ.  PPO, GAE, \qstar{}, and SAO are the more direct comparators.

\newpage
\section{Training Protocol and Computational Cost}
\label{app:data-contracts}
The analysis assumes a fixed sampling law, independent training and validation data, and detached advantages.  The procedure below maintains these conditions within each training window and includes all sampling stages in the cost comparison.

\subsection{Reward scale and checkpoint construction}
\label{app:reward-details}
The reward normalization uses an exponential second moment rather than subtracting a batch mean.  Restoring the window index $n$, let
\begin{equation}
M_{r,n}=\beta M_{r,n-1}+(1-\beta)\frac1{|\mathcal B_n|}
\sum_{\xi\in\mathcal B_n}\Psi_T(\xi)^2,
\qquad \sigma_{r,n}=\sqrt{M_{r,n}+\varepsilon},
\label{eq:reward-scale}
\end{equation}
where $0\leq\beta<1$ and $\varepsilon>0$.  In the calibrated protocol, $\mathcal B_n$ contains historical or disjoint prior trajectories, not head-learning, validation, or actor trajectories from the current window.  The resulting $\sigma_{r,n}$, abbreviated $\sigma_r$ in the main text, remains fixed throughout that window.

Let $0=\tau_0<\tau_1<\cdots<\tau_K=T$ be the scored checkpoints.  Compute $\Psi_{\tau_k}=\ell_0-\ell_{\tau_k}$ and $\Phi_{\tau_k}=\clip(\Psi_{\tau_k}/\sigma_r,-c_r,c_r)$, with $\Phi_0=0$.  Set $\Phi_t=\Phi_{\tau_{k-1}}$ for $\tau_{k-1}<t<\tau_k$.  Equation~\eqref{eq:dense-reward} is then equivalent to
\begin{equation}
r_t=\sum_{k=1}^K\one\{t=\tau_k\}
(\Phi_{\tau_k}-\Phi_{\tau_{k-1}}).
\label{eq:checkpoint-reward}
\end{equation}
The current length-12 configuration uses checkpoints $4$, $8$, and $12$.  In a variable-length rollout, the realized terminal prefix is included as the final checkpoint.  One checkpoint gives a terminal-only shaped reward.  It equals raw utility only without scaling and clipping.

\paragraph{Return conservation within a window.}
\label{app:potential-proof}
Fix one scorer and reward scale.  Include the terminal prefix among checkpoints $0=\tau_0<\cdots<\tau_K=T$, set $\Phi_0=0$, and carry the last potential forward between checkpoints as in Eq.~\eqref{eq:dense-reward}.
\begin{proposition}[Checkpoint decomposition]
\label{prop:checkpoint-decomposition}
Every thought trajectory satisfies
\begin{equation}
\sum_{t=1}^{T}r_t=\Phi_T.
\label{eq:telescoping}
\end{equation}
\end{proposition}
\begin{proof}
Unscored steps contribute zero.  At scored steps, write the sum without cancelling terms first:
\begin{equation}
\sum_t r_t=(\Phi_{\tau_1}-\Phi_0)+(\Phi_{\tau_2}-\Phi_{\tau_1})+\cdots+(\Phi_T-\Phi_{\tau_{K-1}}).
\end{equation}
Every intermediate potential appears once with a plus sign and once with a minus sign.  Only $\Phi_T-\Phi_0$ remains, and $\Phi_0=0$.
\end{proof}
With the raw potential, $\Phi_T=\ell_0-\ell_T=\mathcal U$: intermediate scoring changes where credit is assigned, not its total.  A common positive scale preserves ordering, while clipping merges utilities beyond its boundary.  For example, any two thoughts with $\Psi_T>c_r\sigma_r$ both receive terminal potential $c_r$.  Classical potential shaping has the form $r'(s,a,s')=r(s,a,s')+\gamma\Phi(s')-\Phi(s)$ and preserves optimal policies under fixed potentials and matching discount assumptions \citep{ng1999policy}.  Here the potential depends on the current language model and the scale in Eq.~\eqref{eq:reward-scale}.  We therefore use only the within-iteration telescoping identity.  We do not claim stationary-MDP policy invariance across model updates.

\subsection{Value regression and snapshot references}
\label{app:value-fitting}
Here we restore version and parameter indices omitted in the main text.  For critic $i\in\{1,2\}$, let $V_{i,\psi_i}$ be the trainable model and $V_{i,k}$ its frozen pre-fit prediction in window $k$.  With the return $G_t$ in Eq.~\eqref{eq:critic-return}, the clipped regression objective is
\begin{equation}
\bar V_{i,\psi_i}(s_t)=V_{i,k}(s_t)+
\clip\!\left(V_{i,\psi_i}(s_t)-V_{i,k}(s_t),-\epsilon_v,\epsilon_v\right),
\end{equation}
\begin{equation}
\mathcal L_{V_i}(\psi_i)=\frac12\E_{s_t\sim\mathcal D_k^V}
\!\left[\max\!\left\{(V_{i,\psi_i}(s_t)-G_t)^2,
(\bar V_{i,\psi_i}(s_t)-G_t)^2\right\}\right].
\label{eq:value-loss}
\end{equation}
Critic-fitting outcomes are excluded from predictive holdout, pilot, head-learning, head-validation, and actor trajectory sets.  Split complete text items or parent trajectories before extracting tokens, so correlated tokens are not treated as independent validation samples.  Each selected text-position item contributes one ordinary trajectory in its assigned set, without requiring a bank of repeated-prefix continuations.  Return targets retain their absolute shaped scale.  The run configuration specifies whether full-critic warmup uses the clipped value loss above or ordinary half-squared error.

\subsection{Advantage snapshots and masked GAE}
\label{app:advantage-snapshots}
Let $V_{i,k}^{\mathrm{adv}}$ denote the frozen snapshot used for actor advantages, distinct from the pre-fit value-clipping reference $V_{i,k}$.  This is the post-regression model, frozen before pilot sampling and weight calibration and retained unchanged through head fitting and validation.  The main text writes this snapshot as $V_i$ when constructing actor advantages.  With terminal indicator $d_t$ and realized thought length $L$, the full recurrence is
\begin{align}
\delta_{i,t}&=r_t+\gamma(1-d_t)V_{i,k}^{\mathrm{adv}}(s_{t+1})-V_{i,k}^{\mathrm{adv}}(s_t),\\
A_{i,t}&=\delta_{i,t}+\gamma\lambda(L)(1-d_t)A_{i,t+1},
\qquad \lambda(L)=1-\frac1{\alpha L}.
\label{eq:masked-gae}
\end{align}
The terminal bootstrap and subsequent advantage are zero, and $\alpha$ is chosen so $0\leq\lambda(L)\leq1$.  We use $\gamma=1$ for the actor and full-return targets for critic regression (equivalently $\lambda_V=1$).  The simplified Eq.~\eqref{eq:gae} suppresses $d_t$ under these terminal conventions.

Within window $k$, $\pi_{\mathrm{old}}=\pi_{\theta_k}$ is the recorded denominator, while $q$ is the actual sampling law.  On the valid thought tokens from an independent pilot trajectory set, pool the two raw GAE outputs with equal critic weight.  Set $\mu_{\mathrm{pilot}}$ to their pooled empirical mean and $\sigma_{\mathrm{pilot}}=\sqrt{\widehat\sigma_{\mathrm{pilot}}^2+\varepsilon_A}$ with $\varepsilon_A>0$.  Both channels use these same constants, which remain frozen during weight learning, validation, and PPO.  They are not recomputed after selecting weights or within an actor batch.  Other predetermined common pilot normalizers define a variant that must be reported separately.

\subsection{Trajectory collection and actor replay}
\label{app:restart-contract}
For each selected text-position item, sample one complete thought under the frozen rollout law $q$.  Assign complete items or parent trajectories to critic fitting, predictive holdout, pilot normalization, head learning, head validation, and actor optimization before forming token batches.  Head-learning and actor items need not share prefixes.  All tokens in an actor trajectory are eligible for PPO, with only padding and invalid actions masked.  The complete pre-action context $c_t$ records visible text/thought prefixes, scoring text, time, remaining token allowance, and previous checkpoint potential.  To state the unbiased single-trajectory estimator without a random denominator, fix a maximum thought horizon $T_{\max}$ and let $m_t$ indicate that action $t$ is valid.  Define the finite occupation measure by
\begin{equation}
\E_{d,q}[f]:=\E_{\tau\sim q}\!\left[\frac1{T_{\max}}\sum_{t=1}^{T_{\max}}m_t f(c_t,a_t,\xi_t)\right].
\label{eq:trajectory-occupation}
\end{equation}
For state-only $f$ use $\E_d$.  The validity of position $t$ is known from its prefix, so conditioning on that position preserves the $q$ continuation law.  The measure need not have unit mass, but all risk, dual, and retention terms use the same fixed scale.  A trajectory sum divided by $T_{\max}$ is an unbiased observation.  Population normalization per valid token differs by a fixed positive factor.  Dividing instead by a minibatch's random valid-token count does not have this exact finite-sample unbiasedness.  The auxiliary objective uses the fixed denominator above.  Actor replay can retain its standard valid-token mean because the conditional decomposition is unchanged, but aggregate bounds must specify the occupation normalization.

Compute both GAEs from the same realized trajectory length and reward schedule.  The weight head reads frozen pre-action features and the sampled action, not the realized suffix, return, or GAE.  The auxiliary head reads the full pre-action context, with a separate training-only path for scoring information and no sampled-action input.  GAE labels, both critic encoders, reward scales, and normalizers are detached during head learning.  Optimize Eq.~\eqref{eq:moment-lagrangian}.  Finite replay optimizes an empirical objective and does not supply fresh population gradients at every reuse.

After head fitting, freeze $(\varphi,\zeta)$ before independent validation.  The base retention test checks $\widehat C\leq\kappa\widehat Q$ with the trajectory aggregation above.  A failure falls back to $w=1/2$.  $\widehat Q=0$ is reported as a no-reference-signal diagnostic.  An empirical pass is not a population certificate.  The bounded version in Appendix~\ref{app:moment-validation} can enforce a population retention bound when its assumptions hold.  A second selected candidate requires fresh validation rather than reuse of a rejected validation set as new evidence.

Only after the selected gate is frozen are fresh actor trajectories generated.  Store each valid action, its old log probability, both GAEs, the frozen weight, the detached mixed normalized advantage, and the valid-token mask.  Metadata identify policy, scorer, critic/feature, normalizer, weight, auxiliary, and validation snapshots.  No weight or feature encoder is updated inside the actor PPO window.  Auxiliary language-model and prediction-gate objectives retain separate backward passes.  At window closure refresh the rollout law and collect new data for the changed target.  Single-rollout means one thought per selected item, not one item or trajectory for the entire training process.

\subsection{The prediction-mixture gate}
\label{app:gate-behavior}
The prediction mixture retains its base-language-model anchor and separate gate optimization.  Auxiliary gate losses detach the expert logits and hidden features.  A no-harm term penalizes choices that degrade prediction, and a utility-aware term can tie the mixture weight to signed thought utility.  The final gate layer uses small nonzero weights: an exactly zero layer blocks upstream gradients and can leave the mixture weight constant.

Let $\beta_g(s)\in[0,1]$ be the learned prediction-mixture gate, distinct from the learned critic-mixture weight $w_\varphi(s,a)$ and auxiliary mean function $h_\zeta(c)$.  A utility-aware loss $-\E[\beta_g(s)\mathcal U(s)/\sigma_u]$ has gradient $-\mathcal U(s)\nabla\beta_g(s)/\sigma_u$: it favors thoughts on helpful positions and suppresses them on harmful ones.  A no-harm or mixed-NLL objective can dominate this direction when most utilities are negative.  A negative empirical correlation between gate weight and raw gain therefore does not by itself indicate an incorrect gradient, whereas an exactly blocked output layer prevents adaptation.

\subsection{State, action, and computational cost}
\label{app:computational-cost}
Let $B_{\mathrm{text}}$ be text batch size, $P$ selected positions per sequence, $C$ candidates per position, $T$ thought length, and $K$ potential checkpoints.  The number of hidden actions is $N_{\mathrm{act}}=B_{\mathrm{text}}PCT$.  Thought sampling scales with $N_{\mathrm{act}}$, while continuation scoring scales approximately with $B_{\mathrm{text}}PC(K+1)H$ token evaluations before caching effects.  Exact prefix-value evaluation can additionally scale with the total number of partial prefixes.  The principal memory objects are actor and two critic states, optimizer partitions, replayed action tensors, and temporary future-window logits.  The design does not require a response-generation engine because the rollout is internal and tightly coupled to teacher-forced future evaluation.  Distributed execution instead assigns separate process groups to the actor and critics.  A compact single-node profile devotes most devices to the actor and one device to each critic.  A larger profile uses a multi-node actor group and two independent critic groups.  Replay tensors and model synchronization dominate communication, not external environment interaction.

\subsection{Trajectory counts and computational cost}
\label{app:rollout-accounting}
A rollout is one trajectory generated by the actor from its declared start to termination.  Evaluating that trajectory with two critics produces paired advantages without generating a second trajectory.  The distinction also applies to explicit Monte Carlo mean estimation: from $M$ continuations $(a^{(j)},\xi^{(j)})$ at a fixed prefix, compute
\begin{equation}
\widehat\mu_{\widetilde A_i}(s)
=\frac1M\sum_{j=1}^{M}\widetilde A_i(s,a^{(j)},\xi^{(j)}),
\qquad i\in\{1,2\}.
\label{eq:paired-mean-sampling}
\end{equation}
Both estimates use the same $M$ generated continuations and $2M$ critic evaluations.  They may be correlated, but estimating their individual means does not require separate trajectory sets.  Independent calibration and validation sets, when required by a protocol, add generation for their statistical roles rather than for the number of critics.

For the conditional-moment protocol, each selected item has $C=1$.  Denote the numbers of newly generated trajectories by $n_V$ for critic fitting, $n_{\mathrm{qual}}$ for predictive holdout, $n_{\mathrm{pilot}}$ for normalization, $n_{\mathrm{head}}$ for head learning, $n_{\mathrm{val}}$ for candidate validation, and $n_{\mathrm{act}}$ for actor optimization.  The two critics use the same trajectories in each set.  Let $n_{\mathrm{extra}}$ count additional trajectories from retries, rejected attempts, or diagnostics, excluding those already counted.  Then
\begin{equation}
\begin{aligned}
n_{\mathrm{aux}}&=n_V+n_{\mathrm{qual}}+n_{\mathrm{pilot}}
                  +n_{\mathrm{head}}+n_{\mathrm{val}}+n_{\mathrm{extra}},\\
N_{\mathrm{roll}}^{\mathrm{mom}}&=n_{\mathrm{act}}+n_{\mathrm{aux}}.
\end{aligned}
\label{eq:total-rollout-budget}
\end{equation}
Each newly generated trajectory is counted once.  Replaying stored trajectories adds optimization but no generation.  The original generation is counted in the window in which it occurred.  Historical reuse must still respect the snapshot and separation conditions in Appendix~\ref{app:restart-contract}.  Trajectory counts and head-update counts are configuration parameters, not constants implied by the single-rollout identity.  In particular, equally sized head-training and actor trajectory sets would contribute $2n_{\mathrm{act}}$ trajectories before validation and other trajectory sets are included.

A repeated-prefix calibration baseline with $N_{\mathrm{anchor}}$ anchors and $M_{\mathrm{cal}}$, $M_{\mathrm{adm}}$, and $M_{\mathrm{act}}$ continuations per anchor for calibration, admission, and actor updates generates
\begin{equation}
N_{\mathrm{roll}}^{\mathrm{rep}}
=N_{\mathrm{anchor}}(M_{\mathrm{cal}}+M_{\mathrm{adm}}+M_{\mathrm{act}})
+n_{\mathrm{base}}^{\mathrm{rep}},
\label{eq:repeated-rollout-budget}
\end{equation}
where $n_{\mathrm{base}}^{\mathrm{rep}}$ includes prefix-bank collection, critic fitting, holdout, pilot, and additional attempts.  Choosing $M_{\mathrm{cal}}=M_{\mathrm{adm}}=M$ gives $2M$ calibration/admission continuations per anchor because there are two independent phases, each sharing its trajectories between critics.  The anchor-only variant contributes $N_{\mathrm{anchor}}M_{\mathrm{act}}$ actor actions despite generating their full suffixes.  The moment protocol uses all valid actor-trajectory tokens, but its auxiliary trajectories do not enter the actor loss with the stated data separation.  Its sampling benefit comes from sharing the learned mean function across contexts and using complete actor trajectories.  Estimation error still depends on data coverage, function capacity, and optimization, so this identity alone gives no equal-accuracy sample-complexity guarantee.

For GRPO with $N$ input items and group size $G_{\mathrm{grp}}$, group-response generation contributes $NG_{\mathrm{grp}}$ trajectories, whose valid tokens enter actor training.  At the same number of actor input items, set $n_{\mathrm{act}}=N$.  With no additional GRPO generation in the compared window,
\begin{equation}
N_{\mathrm{roll}}^{\mathrm{mom}}<NG_{\mathrm{grp}}
\quad\Longleftrightarrow\quad
n_{\mathrm{aux}}<(G_{\mathrm{grp}}-1)N.
\label{eq:grpo-rollout-threshold}
\end{equation}
All GRPO-based baselines in our experiments use $G_{\mathrm{grp}}=8$, so this threshold is $n_{\mathrm{aux}}<7N$.  Additional generated data for either method must be included symmetrically.  Equal input-item counts give GRPO more actor trajectories, so this inequality is not an equal actor-data comparison.  At a matched number of actor trajectories, the moment protocol has additional trajectories for auxiliary training and validation, whereas GRPO distributes those actor trajectories into groups.  Neither comparison establishes time or samples needed to reach a target quality.

Report generated-token count as $\sum_{\tau\ \mathrm{generated}}|\tau|$, together with the number of tokens eligible for actor updates.  Repeated-prefix suffixes and full item trajectories can have different lengths, so trajectory counts translate directly to token savings only with comparable average lengths.  Total computation also includes reward scoring, actor updates, both critics' inference and training, feature construction, and weight/mean-head optimization.  Once frozen features and GAE labels are cached, further head passes require no new rollout, but increase computation and optimize the same empirical sample.  Head parameters, optimizer states, and cached features add memory, and constructing the mean head's scoring-context features also costs computation.  Removing prefix replication changes sampling, but speedups depend on total computation and implementation.

\subsection{Refreshing and recovering a training window}
\label{app:window-refresh}
The trainer retains value-head warmup, full-critic warmup, alternating training, and actor-paused recovery.  Backbone unfreezing resets predictive qualification.  Predictive failure pauses actor updates and triggers return-data refresh and critic refitting.  After qualification, freeze critics and normalizers, train the heads, and validate a frozen candidate's signal retention.  Failed retention returns to the feasible fixed mean.  It does not certify a smaller drift.  The base protocol has no hard actor-release test based only on a small saddle loss.  If a drift-certified mode is reported, it must additionally provide all error bounds in Corollary~\ref{cor:admission-budget}.  Changing rollout, scorer, critics, features, or normalizer changes the moment target and requires renewed fitting and independent evaluation.  A whole-window drift claim also needs a common actor sensitivity bound $B$.

\section*{AI Use Statement}
Generative AI tools assisted with manuscript editing, exploration of the conditional-calibration design, and drafting of its analysis.  The authors remain responsible for independently verifying the technical content, claims, citations, and conclusions before submission.

\end{document}